\documentclass[10pt, twocolumn, twoside]{article}

\usepackage[top=2.5cm, bottom=2.5cm, left=2cm, right=2cm]{geometry}

\usepackage{times}           
\usepackage[scaled=0.92]{helvet}  
\usepackage{courier}         

\usepackage[english]{babel}
\usepackage[utf8]{inputenc}
\usepackage[T1]{fontenc}

\usepackage{graphicx}        
\usepackage{epsfig}          
\usepackage{epstopdf}        
\usepackage{subcaption}      
\usepackage{float}           
\usepackage{wrapfig}         
\usepackage{stfloats}
\usepackage{amsmath, amssymb, amsthm, enumitem}
\usepackage{amsfonts} 
\usepackage{bm}              
\usepackage{mathtools}       

\newtheorem{theorem}{Theorem}[section]

\newtheorem{corollary}[theorem]{Corollary}

\theoremstyle{definition}
\newtheorem{definition}[theorem]{Definition}

\newtheorem{remark}[theorem]{Remark}

\usepackage{array}           
\usepackage{booktabs}        
\usepackage{multirow}        
\usepackage{tabularx}        
\usepackage{threeparttable}  

\usepackage{algorithm}       
\usepackage{algorithmic}     
\usepackage{listings}        
\usepackage{color}           

\usepackage[numbers,sort&compress]{natbib}

\usepackage[colorlinks=true, citecolor=blue, linkcolor=blue, urlcolor=blue]{hyperref}
\usepackage[capitalize, nameinlink, noabbrev]{cleveref}  

\usepackage{fancyhdr}
\usepackage{titlesec}
\titleformat{\section}
    {\large\bfseries\sffamily}
    {\thesection}
    {1em}
    {}
\titleformat{\subsection}
    {\normalsize\bfseries\sffamily}
    {\thesubsection}
    {1em}
    {}
\titleformat{\subsubsection}
    {\small\bfseries\sffamily}
    {\thesubsubsection}
    {1em}
    {}

\usepackage{caption}
\usepackage[bottom]{footmisc}  

\usepackage{abstract}

\makeatletter
\renewcommand{\maketitle}{
    \begin{center}
        {\huge\bfseries\sffamily \@title \par}
        \vspace{0.8em}
        {\large \@author \par}
        \vspace{0.5em}
        {\small \@date \par}
    \end{center}
}
\makeatother

\usepackage{enumitem}
\setlist{nosep, leftmargin=*}
\setlist[itemize]{itemsep=2pt, topsep=2pt}
\setlist[enumerate]{itemsep=2pt, topsep=2pt}

\usepackage{url}

\hypersetup{
    pdftitle={Emotion Classification from Text: A Comparative Study},
    pdfauthor={Your Name},
    pdfsubject={Machine Learning},
    pdfkeywords={emotion classification, BERT, BiLSTM, sentiment analysis},
    pdfstartview={FitH}
}

\title{A Statistical Inference Framework for PMI Estimation and SGNS Word Embeddings
    \vspace{1cm}
    
}

\author{
    \textbf{Zhongqi FAN} \\
    Beijing Normal-Hong Kong Baptist University \\
    \texttt{u430033011@mail.bnbu.edu.cn}
}

\date{}

\begin{document}

\twocolumn[
    \begin{@twocolumnfalse}
        \maketitle
        \begin{abstract}
            \noindent
Pointwise Mutual Information (PMI) is a core measure of testing word 
association, and Skip-gram with Negative Sampling (SGNS) is essentially a method that implicitly factorizing a shifted PMI matrix. However, a systematic and well-rounded 
characterization of finite-sample uncertainty in PMI estimation remains 
absent and imperative to venture into.

We provide a statistical framework for PMI estimation and its connection 
to SGNS. We prove consistency, asymptotic unbiasedness, and asymptotic 
normality of the empirical PMI estimator, derive its variance via the 
Delta method, and—applying stochastic approximation theory—obtain a 
variance decomposition for SGNS-based PMI estimation:$\operatorname{Var}(\hat\theta_{\text{SGNS}}) = \frac{e^{-\theta}}{N p_x p_y} + \frac{C(\theta)}{K}$, separating data variance from optimization variance.

Simulation experiments validate the Delta method approximation. 
Real-data experiments on the Brown Corpus ($d=100$) reveal that SGNS 
systematically deviates from the theoretical relationship 
$\hat\theta = \text{PMI} + \log K$ (Levy \& Goldberg, 2014). The empirical 
relationship is $\hat\theta \approx 0.70 \cdot \text{PMI} + 1.57 \cdot \log K + 1.57$: 
the PMI coefficient is attenuated, the $\log K$ effect is amplified 
($\hat\beta \approx 1.57$ vs. $1.00$, $p < 0.001$), and a positive intercept 
($\hat C \approx 1.57$, $p < 0.001$) indicates systematic bias. Word analogy 
validation confirms the models are effective.

The failure to validate the variance decomposition under low-dimensional 
conditions does not diminish and reject its theoretical value; rather, it identifies 
the unbiasedness assumption as the key bottleneck and clarifies the gap 
between asymptotic theory and practice, providing implications for both practice and theory.

        \end{abstract}
        \vspace{0.5cm}
    \end{@twocolumnfalse}
]


\section{Introduction}

\subsection{Background}
Natural Language Processing (NLP) algorithms have achieved great success, especially since the emergence of Large Language Models (LLMs) in recent years. One of the significant components of NLP models is textual data mining or text vectorization. Measurement of the strength of the association between words is the foundation of text vectorization models. Pointwise mutual information (PMI) is an important information measure, which is introduced in the field of computational linguistics as an information theory metric for quantifying word association in 1990\cite{church1990word}.

Research has found that PMI is the implicit targeting of skip-gram with negative-sampling (SGNS)\cite{levy2014neural}, and GloVe and Word2Vec are correlated with factorization of the PMI matrix\cite{qi2024correspondence}. In addition, Another study illustrated that PMI tends to generate inflated and even exaggerated correlation scores in low-frequency word pairs\cite{Daudaravicius2004}, revealing the considerable potential risks of relying solely on PMI point estimates for judgment.

\subsection{Statistic perspective}
In order to conduct further exploration of the stability and generalization of PMI, studying it from a statistical perspective has become a new dimension. The statistician restored the Skip-gram with the Negative Sampling algorithm of Word2Vec using probability theory, clarifying the theoretical connection between this algorithm and PMI\cite{egleston2021statistical}. Czyż et al. proposed the concept of PMI profile the complete distribution of PMI as a random variable, and its expected value is precisely mutual information\cite{czyz2023mixtures}. From the perspective of text bias detection, Valentini et al. proposed a PMI based measurement method and presented a systematically statistical inference framework\cite{valentini-etal-2023-interpretability}.

\subsection{The Present Gap}

Although these studies have advanced our understanding of word embeddings 
and PMI, the \textbf{finite-sample variance structure} of PMI estimators 
and how this uncertainty \textbf{propagates} through SGNS optimization 
has not been systematically and well-rounded characterized. Moreover, the behavior of 
SGNS under \textbf{practical low-dimensional training conditions} remains 
largely unexplored and absent. This paper addresses these gaps by deriving a 
variance decomposition for SGNS-based PMI estimation and empirically 
characterizing its behavior under low-dimensional conditions.

\subsection{Our contribution}
This paper provides a statistical strcture for PMI estimation and its connection to SGNS. We formally prove consistency, asymptotic unbiasedness, and asymptotic normality of empirical PMI, derive a clear variance expression via the Delta method, and extend the analysis to SGNS by decomposing its estimation variance into data sampling and optimization components. Simulation experiments validate the theoretical variance approximation. 

However, real-data experiments on the Brown Corpus show that, under low-dimensional training (d=100), SGNS systematically deviates from theoretical predictions—exhibiting attenuated PMI coefficients, amplified logK effects, and positive bias. Critically, this failed validation does not diminish the theory and reject the academic value; rather, it points out the unbiasedness assumption as the biggest issue, clarifies the gap between asymptotic theory and practice, and gives a baseline for measuring actual deviations.


\section{Methodology}
\label{sec:methods}

\subsection{Preliminary}
We adopt Pointwise Mutual Information (PMI) to quantify the association strength between word pairs\cite{church1990word}.
\begin{definition}\label{PMI}
    For a word pair $(x, y)$, PMI is defined as:
    \[
    \text{PMI}(x,y) = \log \frac{P(x,y)}{P(x)P(y)} 
    \]
    Where $P(x,y)$ is the joint probability of $x$ and $x$, $P(x)$ and $P(y)$ are their marginal probabilities.
\end{definition}
And we define the empirical estimation, as follows
\begin{definition}\label{empirPMI}
The empirical estimation of PMI is
\[
   \widehat{\text{PMI}}(x,y) = \log \frac{\hat P_{xy}}{\hat P_x \hat P_y}
\]
where $\hat P_{xy} = n_{xy} /N$, $\hat P_x = n_{x} /N$, $\hat P_y = n_{y} /N$
\end{definition}

\subsection{Descriptive Statistical Analysis}
To simplify our analysis, assume that each word token is drawn independently from a fixed distribution, and probability$ P_{xy}, P_x, P_y >0$.
\subsubsection{Proof of Consistency}
\begin{theorem}\label{consistency}
    $\widehat{\text{PMI}}(x,y)$ is an estimate of a parameter  $\text{PMI}(x,y)$ based on a sample of size $N$,
    \[
    \widehat{\text{PMI}}(x,y) \xrightarrow{\mathbb{P}} \text{PMI}(x,y)
    \]
    or, more precisely,
    \[
    \mathbb{P}(| \widehat{\text{PMI}} - \text{PMI}| \to 0 \text{ as } N \to \infty
    \]
\end{theorem}

\begin{proof}
    By the law of large numbers (LLN), if sample size $N \to \infty$
    \begin{align*}
        &\hat P_{xy} = n_{xy} /N \xrightarrow{\mathbb{P}} P_{xy}\\
        &\hat P_{x} = n_{x} /N \xrightarrow{\mathbb{P}} P_{x} \\
        &\hat P_{y} = n_{y} /N\xrightarrow{\mathbb{P}} P_{y} \\
    \end{align*}
    Let $g(a,b,c) = \log \frac{a}{bc}$ with $a,b,c \in (0,1)$ be a continuous function. By the continuous mapping theorem for convergence in probability if $X_n \xrightarrow{\mathbb{P}} X$, $g(X_n) \xrightarrow{\mathbb{P}} g(X)$. So,
    \[
    \widehat{\text{PMI}}(x,y) = g(\hat P_{xy}, \hat P_x, \hat P_y) \xrightarrow{\mathbb{P}} g(P_{xy}, P_x,P_y) = \text{PMI}(x,y)
    \]
    Let $\epsilon >0$, Because g is continuous at $(P_{xy}, P_x,P_y)$, $\exists \delta>0$ s.t. if
    \[
    \|(\hat P_{xy} - P_{xy}, \hat P_{x}-P_{x}, \hat P_{y} - P_{y})\| < \delta 
    \]
    then
    \[
     |g(\hat P_{xy}, \hat P_x, \hat P_y) - g(P_{xy}, P_x,P_y)| < \epsilon
    \]
    Thus,
    \[
    \mathbb{P}(| \widehat{\text{PMI}} - \text{PMI}| \ge \epsilon) \le  \mathbb{P}(\|(\hat P_{xy} - P_{xy}, \hat P_{x}-P_{x}, \hat P_{y} - P_{y})\| \ge \delta) 
    \]
    By LLN, for large N, the RHS $\to 0$. This maps Definition 3.2.4 in Consistency.
\end{proof}

\begin{remark}
    We acknowledge that the i.i.d. assumption is violated in natural text due to syntactic dependencies. A more rigorous treatment would require considering document-level dependence. The consistency result presented here serves as a theoretical benchmark.
\end{remark}

\subsubsection{Proof of Asymptotic Unbiasedness}
\begin{theorem}[Asymptotic Unbiasedness of Empirical PMI]
    Let $\widehat{\text{PMI}}(x,y) = \log \frac{\hat{P}_{xy}}{\hat{P}_x \hat{P}_y}$ be the empirical estimator of $\text{PMI}(x,y) = \log \frac{P_{xy}}{P_x P_y}$ based on an i.i.d. sample of size $N$. Assume $P_{xy}, P_x, P_y > 0$. Then
    \[
    \mathbb{E}\big[\widehat{\text{PMI}}(x,y)\big] = \text{PMI}(x,y) + O\!\left(\frac{1}{N}\right),
    \]
    i.e., the bias is of order $1/N$ and vanishes as $N \to \infty$.
\end{theorem}

\begin{proof}
    Define the function $g(a,b,c) = \log a - \log b - \log c$ for $a,b,c \in (0,1]$. Then $\widehat{\text{PMI}} = g(\hat{P}_{xy}, \hat{P}_x, \hat{P}_y)$ and $\text{PMI} = g(P_{xy}, P_x, P_y)$.
    
    \textbf{Step 1 (Second-order Taylor expansion).} 
    Since $g$ is twice continuously differentiable on $(0,1]^3$, we expand around $(P_{xy}, P_x, P_y)$:
    \begin{align*}
        &\widehat{\text{PMI}} - \text{PMI} 
        =\; \frac{\delta_{xy}}{P_{xy}} 
           - \frac{\delta_x}{P_x} 
           - \frac{\delta_y}{P_y} + \\
           & \frac{1}{2} (\delta_{xy},\; \delta_x,\; \delta_y) \; H_g(\xi) \; (\delta_{xy},\; \delta_x,\; \delta_y)^{\mathsf{T}}, \\
        &  \text{where } \delta_{v} = \hat P_{v} - P_{v} \text{with } v\in \{x,y,xy\} 
    \end{align*}
    where $\xi$ lies between $(\hat{P}_{xy},\hat{P}_x,\hat{P}_y)$ and $(P_{xy},P_x,P_y)$, and $H_g$ is the Hessian matrix of $g$.
    
    \textbf{Step 2 (First-order expectation).} 
    For i.i.d. samples, $\mathbb{E}[\hat{P}_{xy}] = P_{xy}$, $\mathbb{E}[\hat{P}_x] = P_x$, $\mathbb{E}[\hat{P}_y] = P_y$. Thus
    \[
    \mathbb{E}\!\left[\frac{\hat{P}_{xy} - P_{xy}}{P_{xy}} - \frac{\hat{P}_x - P_x}{P_x} - \frac{\hat{P}_y - P_y}{P_y}\right] = 0.
    \]
    
    \textbf{Step 3 (Boundedness of the Hessian).} 
    The second derivatives are:
    \[
    \frac{\partial^2 g}{\partial a^2} = -\frac{1}{a^2},\quad
    \frac{\partial^2 g}{\partial b^2} = \frac{1}{b^2},\quad
    \frac{\partial^2 g}{\partial c^2} = \frac{1}{c^2},
    \]
    with all mixed partials zero. Because $P_{xy},P_x,P_y > 0$, there exists a compact set $K \subset (0,1]^3$ containing $(P_{xy},P_x,P_y)$ such that $\mathbb{P}((\hat{P}_{xy},\hat{P}_x,\hat{P}_y) \in K) \to 1$ as $N\to\infty$(by LLN). On $K$, $|1/a^2|,|1/b^2|,|1/c^2| \le M$ for some constant $M>0$. Hence, for some $M'>0$,
    \begin{align*}
    &|(\hat{P}_{xy}-P_{xy},\dots) H_g(\xi) (\hat{P}_{xy}-P_{xy},\dots)^{\mathsf{T}}| \\
    &\le M' \big[ (\hat{P}_{xy}-P_{xy})^2 + (\hat{P}_x-P_x)^2 + (\hat{P}_y-P_y)^2 \big].
    \end{align*}
    
   \textbf{Step 4 (Variance of empirical probabilities).} 
We analyze the variance of $\hat{P}_{xy}$. The cases for $\hat{P}_x$ and $\hat{P}_y$ are analogous.

\underline{Step 4.1 (Define Bernoulli random variables).} 
For fixed $(x,y)$, define for each observation $i = 1,\dots,N$:
\[
Z_i \;=\; \mathbf{1}_{\{X_i = x,\; Y_i = y\}}.
\]
That is, $Z_i = 1$ if the $i$-th sample point equals $(x,y)$, and $Z_i = 0$ otherwise.
Clearly, $Z_i$ takes only the values $\{0,1\}$, hence it is a \textbf{Bernoulli random variable}.

The success probability is:
\[
\mathbb{P}(Z_i = 1) \;=\; \mathbb{P}(X_i = x,\; Y_i = y) \;=\; P_{xy}.
\]

\underline{Step 4.2 (Express $\hat{P}_{xy}$ as a sample mean).} 
\[
\hat{P}_{xy} = \frac{n_{xy}}{N} = \frac{1}{N}\sum_{i=1}^{N} Z_i.
\]

\underline{Step 4.3 (Mean of $Z_i$).} 
\[
\mathbb{E}[Z_i] = 1\cdot P_{xy} + 0\cdot(1-P_{xy}) = P_{xy}.
\]

\underline{Step 4.4 (Variance of $Z_i$).} 
Since $Z_i^2 = Z_i$ (because $0^2=0$, $1^2=1$), we have:
\[
\mathbb{E}[Z_i^2] = \mathbb{E}[Z_i] = P_{xy}.
\]
Therefore:
\[
\operatorname{Var}(Z_i) = \mathbb{E}[Z_i^2] - (\mathbb{E}[Z_i])^2 = P_{xy} - P_{xy}^2 = P_{xy}(1-P_{xy}).
\]

\underline{Step 4.5 (Variance of $\hat{P}_{xy}$ under i.i.d. sampling).} 
Because the sample is i.i.d., $Z_i$ are independent. Using the property $\operatorname{Var}(cX) = c^2 \operatorname{Var}(X)$ for a constant $c$, and the fact that the variance of a sum of independent variables equals the sum of their variances:
\begin{align*}
\operatorname{Var}(\hat{P}_{xy}) &= \operatorname{Var}\!\left(\frac{1}{N}\sum_{i=1}^{N} Z_i\right) \\
&= \frac{1}{N^2} \sum_{i=1}^{N} \operatorname{Var}(Z_i) \\
&= \frac{1}{N^2} \cdot N \cdot P_{xy}(1-P_{xy}) \\
&= \frac{P_{xy}(1-P_{xy})}{N}.\\
\end{align*}

\underline{Step 4.6 (Mean squared error).} 
Since $\mathbb{E}[\hat{P}_{xy}] = P_{xy}$, the bias is zero, so the mean squared error equals the variance:
\[
\mathbb{E}[(\hat{P}_{xy} - P_{xy})^2] = \operatorname{Var}(\hat{P}_{xy}) = \frac{P_{xy}(1-P_{xy})}{N}.
\]

\underline{Step 4.7 (Order notation).} 
For fixed $P_{xy} \in (0,1)$, the quantity $P_{xy}(1-P_{xy})$ is a constant (specifically, $0 < P_{xy}(1-P_{xy}) \le 1/4$). Hence:
\[
\mathbb{E}[(\hat{P}_{xy} - P_{xy})^2] = O\!\left(\frac{1}{N}\right) \quad \text{as } N\to\infty.
\]

\underline{Step 4.8 (Analogous results for $\hat{P}_x$ and $\hat{P}_y$).} 
The same reasoning applied to $U_i = \mathbf{1}_{\{X_i = x\}}$ and $V_i = \mathbf{1}_{\{Y_i = y\}}$ yields:
\[
\mathbb{E}[(\hat{P}_x - P_x)^2] = \frac{P_x(1-P_x)}{N} = O\!\left(\frac{1}{N}\right),
\]
\[
\mathbb{E}[(\hat{P}_y - P_y)^2] = \frac{P_y(1-P_y)}{N} = O\!\left(\frac{1}{N}\right).
\]
    
    \textbf{Step 5 (Bound on the remainder expectation).} 
    Let $R_2$ denote the quadratic remainder term. Then
    \begin{align*}
        |\mathbb{E}[R_2]| &\le M' \left( \mathbb{E}[(\hat{P}_{xy}-P_{xy})^2] + \mathbb{E}[(\hat{P}_x-P_x)^2] + \mathbb{E}[(\hat{P}_y-P_y)^2] \right)\\
        &= O\!\left(\frac{1}{N}\right).\\
    \end{align*}
    
    \textbf{Step 6 (Combine).} 
    Taking expectations of the expansion,
    \[
    \mathbb{E}[\widehat{\text{PMI}}] - \text{PMI} = 0 + \frac12 \mathbb{E}[R_2] = O\!\left(\frac{1}{N}\right).
    \]
    
    Therefore,
    \[
    \mathbb{E}[\widehat{\text{PMI}}(x,y)] = \text{PMI}(x,y) + O\!\left(\frac{1}{N}\right),
    \]
    which implies $\lim_{N\to\infty} \mathbb{E}[\widehat{\text{PMI}}(x,y)] = \text{PMI}(x,y)$.
\end{proof}

\subsubsection{Proof of Asymptotic Normality}
\begin{theorem}[Asymptotic Normality of Empirical PMI]
Let $(X_1,Y_1),\dots,(X_N,Y_N)$ be i.i.d. copies of $(X,Y)$.
Then as $N\to\infty$,
\[
\sqrt{N}\Bigl(\widehat{\operatorname{PMI}}(x,y)-\operatorname{PMI}(x,y)\Bigr)\xrightarrow{d}\mathcal{N}\bigl(0,\;\sigma^2(x,y)\bigr),
\]
or equivalently,
\[
\widehat{\operatorname{PMI}}(x,y)\;\xrightarrow{d}\;\mathcal{N}\!\left(\operatorname{PMI}(x,y),\;\frac{\sigma^2(x,y)}{N}\right).
\]
The asymptotic variance $\sigma^2(x,y)$ is given by $\sigma^2 = \nabla g^{\mathsf{T}}\Sigma\nabla g$ with $g(a,b,c)=\log a-\log b-\log c$ and $\Sigma$ defined below.
\end{theorem}

\begin{proof}
\textbf{Step 1. Bernoulli variables and sample means.}
For each $i=1,\dots,N$, define
\[
A_i=\mathbf{1}_{\{X_i=x,\,Y_i=y\}},\quad 
B_i=\mathbf{1}_{\{X_i=x\}},\quad 
C_i=\mathbf{1}_{\{Y_i=y\}}.
\]
These are Bernoulli with expectations $\mathbb{E}[A_i]=p_{xy}$, $\mathbb{E}[B_i]=p_x$, $\mathbb{E}[C_i]=p_y$. The empirical frequencies are the sample means:
\[
\hat{p}_{xy}=\frac{1}{N}\sum_{i=1}^N A_i,\quad
\hat{p}_x=\frac{1}{N}\sum_{i=1}^N B_i,\quad
\hat{p}_y=\frac{1}{N}\sum_{i=1}^N C_i.
\]
\textbf{Step 2. Multivariate CLT for the vector of means.}
Set $\mathbf{Z}_i=(A_i,B_i,C_i)^{\mathsf{T}}$. Then $\mathbf{Z}_1,\dots,\mathbf{Z}_N$ are i.i.d. with mean $\boldsymbol{\mu}=(p_{xy},p_x,p_y)^{\mathsf{T}}$ and covariance matrix $\Sigma=\operatorname{Cov}(\mathbf{Z}_1)$. By the Multivariate Central Limit Theorem (see \cite{billingsley2012probability}, Theorem 29.5),
\[
\sqrt{N}\bigl(\hat{\boldsymbol{\mu}}_N-\boldsymbol{\mu}\bigr)\xrightarrow{d}\mathcal{N}_3(\mathbf{0},\Sigma),\qquad 
\hat{\boldsymbol{\mu}}_N=(\hat{p}_{xy},\hat{p}_x,\hat{p}_y)^{\mathsf{T}}.
\]
Compute the entries of $\Sigma$:
\begin{align*}
\operatorname{Var}(A_i)&=p_{xy}(1-p_{xy}),\\
\operatorname{Var}(B_i)&=p_x(1-p_x),\\
\operatorname{Var}(C_i)&=p_y(1-p_y),\\
\operatorname{Cov}(A_i,B_i)&=\mathbb{E}[A_iB_i]-p_{xy}p_x=p_{xy}-p_{xy}p_x=p_{xy}(1-p_x),\\
\operatorname{Cov}(A_i,C_i)&=p_{xy}(1-p_y),\\
\operatorname{Cov}(B_i,C_i)&=\mathbb{E}[B_iC_i]-p_xp_y=p_{xy}-p_xp_y.
\end{align*}

The explicit form of $\Sigma$ is
\begin{equation}\label{eq:Sigma}
\Sigma=\begin{pmatrix}
p_{xy}(1-p_{xy}) & p_{xy}(1-p_x) & p_{xy}(1-p_y)\\
p_{xy}(1-p_x) & p_x(1-p_x) & p_{xy}-p_xp_y\\
p_{xy}(1-p_y) & p_{xy}-p_xp_y & p_y(1-p_y)
\end{pmatrix}.
\end{equation}
\textbf{Step 3. First-order Taylor expansion of $g$.}
Define $g(a,b,c)=\log a-\log b-\log c$. Then $\widehat{\operatorname{PMI}}=g(\hat{p}_{xy},\hat{p}_x,\hat{p}_y)$ and $\operatorname{PMI}=g(p_{xy},p_x,p_y)$. Since $g$ is continuously differentiable on $(0,1]^3$, we expand it around $\boldsymbol{\mu}$:
\begin{align}
g(\hat{\boldsymbol{\mu}}_N) &= g(\boldsymbol{\mu}) + \nabla g(\boldsymbol{\mu})^{\mathsf{T}} (\hat{\boldsymbol{\mu}}_N-\boldsymbol{\mu}) + R_N, \label{eq:Taylor}
\end{align}
where the gradient is
\[
\nabla g(\boldsymbol{\mu}) = \begin{pmatrix}
\frac{\partial g}{\partial a}\big|_{\boldsymbol{\mu}} \\[4pt]
\frac{\partial g}{\partial b}\big|_{\boldsymbol{\mu}} \\[4pt]
\frac{\partial g}{\partial c}\big|_{\boldsymbol{\mu}}
\end{pmatrix}
= \begin{pmatrix}
\frac{1}{p_{xy}} \\[4pt]
-\frac{1}{p_x} \\[4pt]
-\frac{1}{p_y}
\end{pmatrix},
\]
and the remainder term $R_N$ satisfies $R_N = o_p(\|\hat{\boldsymbol{\mu}}_N-\boldsymbol{\mu}\|)$.

\textbf{Step 4. Multiply by $\sqrt{N}$ and analyze the remainder.}
Multiplying \eqref{eq:Taylor} by $\sqrt{N}$ gives
\[
\sqrt{N}\bigl(g(\hat{\boldsymbol{\mu}}_N)-g(\boldsymbol{\mu})\bigr) = \nabla g(\boldsymbol{\mu})^{\mathsf{T}} \bigl[\sqrt{N}(\hat{\boldsymbol{\mu}}_N-\boldsymbol{\mu})\bigr] + \sqrt{N}R_N.
\]
We claim $\sqrt{N}R_N \xrightarrow{p} 0$. Indeed, from the fact that $g$ is locally Lipschitz with bounded second derivatives, we have $|R_N| \le C \|\hat{\boldsymbol{\mu}}_N-\boldsymbol{\mu}\|^2$ for some constant $C$ (with probability tending to 1). Hence
\[
|\sqrt{N}R_N| \le C \sqrt{N} \|\hat{\boldsymbol{\mu}}_N-\boldsymbol{\mu}\|^2 = C \bigl(\sqrt{N}\|\hat{\boldsymbol{\mu}}_N-\boldsymbol{\mu}\|\bigr) \|\hat{\boldsymbol{\mu}}_N-\boldsymbol{\mu}\|.
\]
Since $\sqrt{N}\|\hat{\boldsymbol{\mu}}_N-\boldsymbol{\mu}\| = O_p(1)$ and $\|\hat{\boldsymbol{\mu}}_N-\boldsymbol{\mu}\| = o_p(1)$, their product converges to $0$ in probability.

\textbf{Step 5. Convergence of the linear term.}
We already know from Step 2 that
\[
\sqrt{N}(\hat{\boldsymbol{\mu}}_N-\boldsymbol{\mu}) \xrightarrow{d} \mathbf{Z},\quad \mathbf{Z}\sim\mathcal{N}_3(\mathbf{0},\Sigma).
\]
Since $\nabla g(\boldsymbol{\mu})$ is a fixed vector, the linear map $\mathbf{v} \mapsto \nabla g(\boldsymbol{\mu})^{\mathsf{T}}\mathbf{v}$ is continuous. By the continuous mapping theorem,
\[
\nabla g(\boldsymbol{\mu})^{\mathsf{T}}\bigl[\sqrt{N}(\hat{\boldsymbol{\mu}}_N-\boldsymbol{\mu})\bigr] \xrightarrow{d} \nabla g(\boldsymbol{\mu})^{\mathsf{T}} \mathbf{Z}.
\]

\textbf{Step 6. Distribution of the linear combination.}
A linear combination of a multivariate normal vector is a univariate normal:
\[
\nabla g(\boldsymbol{\mu})^{\mathsf{T}} \mathbf{Z} \sim \mathcal{N}\bigl(0,\; \nabla g(\boldsymbol{\mu})^{\mathsf{T}} \Sigma \nabla g(\boldsymbol{\mu})\bigr).
\]
Indeed,
\[
\mathbb{E}[\nabla g^{\mathsf{T}}\mathbf{Z}] = \nabla g^{\mathsf{T}}\mathbb{E}[\mathbf{Z}] = 0,
\]
\[
\operatorname{Var}(\nabla g^{\mathsf{T}}\mathbf{Z}) = \nabla g^{\mathsf{T}} \operatorname{Var}(\mathbf{Z}) \nabla g = \nabla g^{\mathsf{T}} \Sigma \nabla g.
\]
\textbf{Step 7. Combine with the remainder.}
Since $\sqrt{N}R_N \xrightarrow{p} 0$, 
\[
\sqrt{N}\bigl(g(\hat{\boldsymbol{\mu}}_N)-g(\boldsymbol{\mu})\bigr) \xrightarrow{d} \mathcal{N}\bigl(0,\; \nabla g(\boldsymbol{\mu})^{\mathsf{T}} \Sigma \nabla g(\boldsymbol{\mu})\bigr).
\]
\textbf{Step 8. Conclude for PMI.}
Substituting $g(\hat{\boldsymbol{\mu}}_N)=\widehat{\operatorname{PMI}}(x,y)$ and $g(\boldsymbol{\mu})=\operatorname{PMI}(x,y)$ gives
\[
\sqrt{N}\bigl(\widehat{\operatorname{PMI}}(x,y)-\operatorname{PMI}(x,y)\bigr) \xrightarrow{d} \mathcal{N}\bigl(0,\;\sigma^2(x,y)\bigr),
\]
with $\sigma^2(x,y)=\nabla g(\boldsymbol{\mu})^{\mathsf{T}}\Sigma\nabla g(\boldsymbol{\mu})$. Equivalently,
\[
\widehat{\operatorname{PMI}}(x,y)\;\xrightarrow{d}\;\mathcal{N}\!\left(\operatorname{PMI}(x,y),\;\frac{\sigma^2(x,y)}{N}\right).
\]
This completes the proof.
\end{proof}

\subsubsection{Delta Method Estimate Variance}

\begin{theorem}[Asymptotic Variance of Empirical PMI]\label{var}
Let $\widehat{\operatorname{PMI}}(x,y)=\log\frac{\hat{p}_{xy}}{\hat{p}_x\hat{p}_y}$ be the empirical estimator of $\operatorname{PMI}(x,y)=\log\frac{p_{xy}}{p_xp_y}$ based on an i.i.d. sample of size $N$, with $p_{xy},p_x,p_y>0$. Then as $N\to\infty$, an approximate finite-sample expression is
\[
\widehat{\operatorname{Var}}\bigl(\widehat{\operatorname{PMI}}(x,y)\bigr) = \frac{1}{n_{xy}} - \frac{1}{n_x} - \frac{1}{n_y} + \frac{2n_{xy}}{n_x n_y} - \frac{1}{N},
\]
which is consistent with the asymptotic variance up to $o_p(1/N)$.
\end{theorem}

\begin{proof}
From the Delta method\cite{casella2002statistical}[Theorem 5.5.24] (as derived in the proof of asymptotic normality), we have
\[
\sqrt{N}\bigl(\widehat{\operatorname{PMI}} - \operatorname{PMI}\bigr) \xrightarrow{d} \mathcal{N}(0,\sigma^2),
\]
where $\sigma^2 = \nabla g(\boldsymbol{\mu})^{\mathsf{T}} \Sigma \nabla g(\boldsymbol{\mu})$ with
\[
g(a,b,c)=\log a-\log b-\log c,\quad
\boldsymbol{\mu}=(p_{xy},p_x,p_y)^{\mathsf{T}},
\]
\[
\nabla g(\boldsymbol{\mu}) = \begin{pmatrix} 1/p_{xy} \\ -1/p_x \\ -1/p_y \end{pmatrix},
\]
and
\[
\Sigma = \begin{pmatrix}
p_{xy}(1-p_{xy}) & p_{xy}(1-p_x) & p_{xy}(1-p_y)\\
p_{xy}(1-p_x) & p_x(1-p_x) & p_{xy}-p_xp_y\\
p_{xy}(1-p_y) & p_{xy}-p_xp_y & p_y(1-p_y)
\end{pmatrix}.
\]

Compute $\sigma^2$ explicitly:
\begin{align*}
\sigma^2 =&\ \frac{1}{p_{xy}^2}\cdot p_{xy}(1-p_{xy}) 
+ \frac{1}{p_x^2}\cdot p_x(1-p_x) 
+ \frac{1}{p_y^2}\cdot p_y(1-p_y) \\
&+ 2\cdot\frac{1}{p_{xy}}\left(-\frac{1}{p_x}\right)\cdot p_{xy}(1-p_x)  \\
&+ 2\cdot\frac{1}{p_{xy}}\left(-\frac{1}{p_y}\right)\cdot p_{xy}(1-p_y) \\
&+ 2\cdot\left(-\frac{1}{p_x}\right)\left(-\frac{1}{p_y}\right)\cdot (p_{xy}-p_xp_y) \\[4pt]
=&\ \frac{1-p_{xy}}{p_{xy}} + \frac{1-p_x}{p_x} + \frac{1-p_y}{p_y} \\
&- \frac{2(1-p_x)}{p_x} - \frac{2(1-p_y)}{p_y} + \frac{2(p_{xy}-p_xp_y)}{p_xp_y}.
\end{align*}

Simplify:
\[
\frac{1-p_{xy}}{p_{xy}} = \frac{1}{p_{xy}} - 1,\quad
\frac{1-p_x}{p_x} = \frac{1}{p_x} - 1,\quad
\frac{1-p_y}{p_y} = \frac{1}{p_y} - 1.
\]
Thus
\begin{align*}
\sigma^2 &= \left(\frac{1}{p_{xy}}-1\right) + \left(\frac{1}{p_x}-1\right) + \left(\frac{1}{p_y}-1\right) \\
&\quad -2\left(\frac{1}{p_x}-1\right) -2\left(\frac{1}{p_y}-1\right) + \frac{2(p_{xy}-p_xp_y)}{p_xp_y} \\[4pt]
&= \frac{1}{p_{xy}} - 1 - \frac{1}{p_x} + 1 - \frac{1}{p_y} + 1 + \frac{2(p_{xy}-p_xp_y)}{p_xp_y} \\
&= \frac{1}{p_{xy}} - \frac{1}{p_x} - \frac{1}{p_y} + \frac{2p_{xy}}{p_xp_y} - 1.
\end{align*}

Since $\operatorname{Var}(\widehat{\operatorname{PMI}}) = \frac{\sigma^2}{N} + o\left(\frac{1}{N}\right)$ (by the convergence of moments under uniform integrability, or simply as the leading term of the asymptotic variance), we obtain the asymptotic variance expression.

Replacing $p_{xy},p_x,p_y$ by their empirical counterparts $n_{xy}/N$, $n_x/N$, $n_y/N$ yields the approximate finite-sample formula:
\[
\widehat{\operatorname{Var}}(\widehat{\operatorname{PMI}}) = \frac{1}{n_{xy}} - \frac{1}{n_x} - \frac{1}{n_y} + \frac{2n_{xy}}{n_x n_y} - \frac{1}{N}.
\]
This completes the proof.
\end{proof}

\begin{corollary}\label{se}
    If $n_{xy} << n_x, n_y$,
    The standard error (or standard deviation) of $\widehat{\text{PMI}}(x,y)$ is approximately linear with sample size $N$ in logarithmic scale, that is
    \[
    \log(SE) \approx - \frac{1}{2} \log N - \frac{\text{PMI}}{2} - \frac{1}{2} \log(p_xp_y)
    \]
\end{corollary}
\begin{proof}
    By Delta method result, 
     \[
    \text{Var}(\widehat{\text{PMI}}) \approx \frac{1}{n_{xy}} - \frac{1}{n_x} - \frac{1}{n_y} + \frac{2n_{xy}}{n_x n_y} - \frac{1}{N}
    \]
    Because $n_{xy} << n_x, n_y$, the dominant term is $\frac{1}{n_{xy}}$.
    Thu $\text{Var}(\widehat{\text{PMI}}) \approx \frac{1}{n_{xy}}$
    By definition of PMI we have,
    \[
    \text{PMI}(x,y) = \log \frac{p_{xy}}{p_x p_y} \approx \log \frac{n_{xy}}{Np_x p_y}
    \]
    That is,
    \[
    n_{xy} \approx Np_x p_y \cdot e^{\text{PMI}}  
    \]
    substituting to $\text{Var}(\widehat{\text{PMI}})$,
    \[
    \text{Var}(\widehat{\text{PMI}}) \approx \frac{1}{Np_x p_y}\cdot e^{-\text{PMI}}  
    \]
    Therefore, take the derivative and logarithmic provide
    \[
    \log(SE) \approx - \frac{1}{2} \log N - \frac{\text{PMI}}{2} - \frac{1}{2} \log(p_xp_y)
    \]
\end{proof}

\subsection{SGNS with PMI}
SGNS\cite{levy2014neural} is an essential model in text vectorization, and it is proved that the essence of optimization objective in SGNS is matrix factorization of PMI \cite{levy2014neural}. When the vector dimension is sufficiently large to allow perfect reconstruction, the optimal solution satisfies:
\[
\boxed{\mathbf{v}_x \cdot \mathbf{v}_y' = \text{PMI}(x,y) - \log K} 
\]

\subsection{Variance Transmission}

\textbf{Notation.}

\begin{table}[H]
\centering
\small
\begin{tabular}{@{} ll } 
$x, y$ & Specific words in the vocabulary \\
$V$ & Vocabulary size \\
$N$ & Total number of windows in the corpus \\
$n_x$ & Number of windows containing word $x$ \\
$n_{xy}$ & Number of windows word $x$ and $y$  co-occurring\\
$p_x = n_x/N$ & Empirical probability of word $x$ \\
$p_{xy} = n_{xy}/N$ & Empirical probability of $x$ and $y$ co-occurring  \\
$\theta = \text{PMI}(x,y) $ & True pointwise mutual information \\
$\hat{\theta}_{\text{MLE}}$ & Maximum likelihood estimator of PMI \\
$\mathbf{v}_x \in \mathbb{R}^d$ & Embedding vector for $x$ as center word \\
$\mathbf{v}_x' \in \mathbb{R}^d$ & Embedding vector for $x$ as context word \\
$s = \mathbf{v}_x \cdot \mathbf{v}_y'$ & Inner product learned by SGNS \\
$K$ & Number of negative samples per word pair \\
$P_n(c)$ & Noise distribution for negative sampling \\
$\eta_t$ & Learning rate at step $t$ \\
$\bm{\phi}$ & Vector of all model parameters $\{\mathbf{v}_w, \mathbf{v}_w'\}_{w=1}^V$ \\
$\mathcal{L}(\bm{\phi})$ & Full-batch loss function \\
$\hat{\mathbf{g}}(\bm{\phi}, \xi)$ & SGD estimate based on random sample $\xi$ \\
$\mathbf{H} = \nabla^2 \mathcal{L}(\bm{\phi}^*)$ & Hessian matrix at optimum \\
$\mathbf{C} = \lim\limits_{t\to\infty} \text{Cov}(\bm{\varepsilon}_t)$ & Asymptotic covariance of gradient noise \\
\end{tabular}
\end{table}
\paragraph{Assumption}

The derivation rests on the following assumptions:

\begin{enumerate}
    \item \textbf{(A1) Asymptotic independence:} For distinct word pairs $(x,y)$ and $(x',y')$, the PMI estimates $\hat{\theta}_{xy}$ and $\hat{\theta}_{x'y'}$ are asymptotically independent as $N \to \infty$. This holds under multinomial sampling when the pairs do not share words; for pairs sharing words, the dependence is weak and does not affect the first-order variance analysis.

    \item \textbf{(A2) SGD convergence:} The stochastic gradient descent algorithm has reached a stationary neighborhood of the optimal solution $\bm{\phi}^*$. We analyze the steady-state behavior.

    \item \textbf{(A3) Learning rate condition:} The learning rate $\eta_t$ satisfies the Robbins-Monro conditions\cite{robbins1951stochastic}: $\sum_{t=1}^\infty \eta_t = \infty$ and $\sum_{t=1}^\infty \eta_t^2 < \infty$. For constant learning rate (common in practice), the analysis applies to the stationary distribution.

    \item \textbf{(A4) Regularity of loss function:} The SGNS loss function $\mathcal{L}(\bm{\phi})$ is twice continuously differentiable in a neighborhood of $\bm{\phi}^*$, and the Hessian $\mathbf{H} = \nabla^2 \mathcal{L}(\bm{\phi}^*)$ is positive definite.

    \item \textbf{(A5) Gradient noise structure:} The covariance matrix of the stochastic gradient noise $\bm{\varepsilon}(\bm{\phi},\xi) = \hat{\mathbf{g}}(\bm{\phi},\xi) - \nabla \mathcal{L}(\bm{\phi})$ exists and is finite. Moreover, it decomposes additively into contributions from positive sampling and negative sampling.

    \item \textbf{(A6) Fixed noise distribution:} The negative sampling distribution $P_n(c)$ is fixed and independent of the model parameters. This holds by construction in standard SGNS implementations \cite{mikolov2013distributed}.
\end{enumerate}

\section*{Step 1: Variance of the PMI Maximum Likelihood Estimator}

For a fixed word pair $(x,y)$, the empirical PMI estimator is defined as \cite{church1990word}:
\[
\hat{\theta}_{\text{MLE}} = \log \frac{n_{xy} N}{n_x n_y}.
\]

Under multinomial sampling, the Delta method \cite{casella2002statistical} yields the asymptotic variance. The exact finite-sample variance derived from second-order Taylor expansion is:
\[
\boxed{\operatorname{Var}(\hat{\theta}_{\text{MLE}}) = \frac{1}{n_{xy}} - \frac{1}{n_x} - \frac{1}{n_y} + \frac{2n_{xy}}{n_x n_y} - \frac{1}{N}} \tag{1}
\]

For large $N$, the dominant terms simplify.  By asymptotic standard deviation corollary\ref{se}, 
\[
\sigma_{\hat \theta} \approx \frac{1}{\sqrt{N}} \cdot e^{-\theta/2} \cdot \frac{1}{\sqrt{p_x p_y}}
\]
Take the square, we obtain the asymptotic approximation:
\[
\boxed{\operatorname{Var}(\hat{\theta}_{\text{MLE}}) \approx \frac{e^{-\theta}}{N p_x p_y}} \tag{2}
\]

Equation (2) reveals that the variance decreases exponentially with $\theta$: strongly positively associated word pairs ($\theta \gg 0$) have much smaller estimation variance than negatively associated or weakly associated pairs.

\section*{Step 2: Optimal Solution of SGNS}

Skip-gram with Negative Sampling (SGNS) \cite{mikolov2013distributed} optimizes the following objective for each positive pair $(x,y)$ with $K$ independently sampled negative examples $y_1,\dots,y_K \sim P_n(y)$:
\[
\mathcal{L}(x,y) = \log \sigma(\mathbf{v}_x \cdot \mathbf{v}_y') + \sum_{k=1}^K \log \sigma(-\mathbf{v}_x \cdot \mathbf{v}_{y_k}').
\]

Levy and Goldberg \cite{levy2014neural} proved that when the vector dimension is sufficiently large to allow perfect reconstruction, the optimal solution satisfies:
\[
\boxed{\mathbf{v}_x \cdot \mathbf{v}_y' = \theta - \log K} \tag{3}
\]

Consequently, we can define a PMI estimator based on SGNS:
\[
\hat{\theta}_{\text{SGNS}} = \mathbf{v}_x \cdot \mathbf{v}_y' + \log K \tag{4}
\]

Under the optimal solution, $\hat{\theta}_{\text{SGNS}} = \theta$.

\section*{Step 3: Sources of Error from Stochastic Gradient Descent}

In practice, SGNS is trained using SGD. The parameter update at step $t$ is:
\[
\bm{\phi}_{t+1} = \bm{\phi}_t - \eta_t \hat{\mathbf{g}}(\bm{\phi}_t, \xi_t),
\]
where $\xi_t$ represents the random sample (one positive pair plus $K$ negative examples) at step $t$.

The stochastic gradient decomposes into the true gradient plus noise:
\[
\hat{\mathbf{g}}(\bm{\phi}, \xi) = \nabla \mathcal{L}(\bm{\phi}) + \bm{\varepsilon}(\bm{\phi}, \xi), \qquad \mathbb{E}[\bm{\varepsilon}] = \mathbf{0}.
\]

Three distinct sources contribute to the gradient noise $\bm{\varepsilon}$:
\begin{enumerate}
    \item \textbf{Positive sampling noise}: The choice of which positive pair $(x,y)$ to use at each step.
    \item \textbf{Negative sampling noise}: The random selection of $K$ negative examples.
    \item \textbf{Initialization and trajectory effects}: The random starting point and subsequent stochastic path.
\end{enumerate}

For the covariance of $\bm{\varepsilon}$, under the fixed noise distribution (A6) we have the decomposition:
\[
\boxed{\operatorname{Cov}(\bm{\varepsilon}) = \mathbf{C}_{\text{pos}} + K \cdot \mathbf{C}_{\text{neg}}} \tag{5}
\]

\textit{Justification:} The positive sampling contribution $\mathbf{C}_{\text{pos}}$ is independent of $K$ because exactly one positive pair is used per update. The negative sampling contribution sums over $K$ independent negative examples, each contributing variance $\mathbf{C}_{\text{neg}}$; by independence of the $K$ draws, the total variance scales linearly with $K$.

\section*{Step 4: Asymptotic Variance of SGD Iterates}

A fundamental result in stochastic approximation \cite{polyak1992acceleration}[Theorem 3] states that under assumptions (A2)-(A5), the averaged SGD iterates satisfy a central limit theorem:
\[
\sqrt{t}(\bar{\bm{\phi}}_t - \bm{\phi}^*) \xrightarrow{d} \mathcal{N}(\mathbf{0}, \mathbf{\Sigma}_{\text{asy}}),
\]
where $\bar{\bm{\phi}}_t = \frac{1}{t}\sum_{i=1}^t \bm{\phi}_i$ and the asymptotic covariance matrix is:
\[
\boxed{\mathbf{\Sigma}_{\text{asy}} = \mathbf{H}^{-1} \mathbf{C} \mathbf{H}^{-1}} \tag{6}
\]

Here $\mathbf{H} = \nabla^2 \mathcal{L}(\bm{\phi}^*)$ is the Hessian at the optimum, and $\mathbf{C}$ is the asymptotic covariance of the gradient noise $\bm{\varepsilon}$.

\section*{Step 5: Scaling of Hessian with $K$}

For the SGNS loss, the Hessian also depends on $K$. The loss function is a sum over positive and negative pairs\cite{mikolov2013distributed}:
\[
\mathcal{L}(\bm{\phi}) = \sum_{(x,y)\in\mathcal{D}_{\text{pos}}} \left[ \log \sigma(\mathbf{v}_x \cdot \mathbf{v}_y') + K \cdot \mathbb{E}[ \log \sigma(-\mathbf{v}_x \cdot \mathbf{v}_{y'}')] \right].
\]

Taking second derivatives, the Hessian decomposes as:
\[
\boxed{\mathbf{H} = \mathbf{H}_{\text{pos}} + K \cdot \mathbf{H}_{\text{neg}}} \tag{7}
\]

where $\mathbf{H}_{\text{pos}}$ is the contribution from positive pairs (curvature around the true co-occurrence structure) and $\mathbf{H}_{\text{neg}}$ is the average curvature from a single negative sample. By assumption (A4), $\mathbf{H}_{\text{neg}}$ is positive definite.

\section*{Step 6: Deriving the $1/K$ Scaling}

Substitute (5) and (7) into the asymptotic variance formula (6):
\[
\mathbf{\Sigma}_{\text{asy}}(K) = (\mathbf{H}_{\text{pos}} + K\mathbf{H}_{\text{neg}})^{-1} (\mathbf{C}_{\text{pos}} + K\mathbf{C}_{\text{neg}}) (\mathbf{H}_{\text{pos}} + K\mathbf{H}_{\text{neg}})^{-1}.
\]

\textbf{Case 1: Small $K$ (negligible negative sampling).}
When $K \to 0$, we obtain the baseline variance:
\[
\mathbf{\Sigma}_{\text{asy}}(0) = \mathbf{H}_{\text{pos}}^{-1} \mathbf{C}_{\text{pos}} \mathbf{H}_{\text{pos}}^{-1} \equiv \mathbf{\Sigma}_0.
\]

\textbf{Case 2: Large $K$ (dominant negative sampling).}
When $K$ is sufficiently large that $K\mathbf{H}_{\text{neg}}$ dominates $\mathbf{H}_{\text{pos}}$ and $K\mathbf{C}_{\text{neg}}$ dominates $\mathbf{C}_{\text{pos}}$, we have:
\begin{align*}
    \mathbf{\Sigma}_{\text{asy}}(K) &\approx (K\mathbf{H}_{\text{neg}})^{-1} (K\mathbf{C}_{\text{neg}}) (K\mathbf{H}_{\text{neg}})^{-1}\\ 
    &= \frac{1}{K} \cdot \mathbf{H}_{\text{neg}}^{-1} \mathbf{C}_{\text{neg}} \mathbf{H}_{\text{neg}}^{-1} \\
    &= \frac{1}{K} \cdot \mathbf\Sigma_\infty
\end{align*}

Thus, for large $K$, the asymptotic variance scales as $1/K$. This is the key theoretical insight.

\section*{Step 7: From Parameter Covariance to PMI Variance}

The asymptotic covariance of the parameter estimates is given by Equation (8):
\[
\boldsymbol{\Sigma}_{\text{asy}}(K) = \boldsymbol{\Sigma}_0 + \frac{\boldsymbol{\Sigma}_{\infty}}{K} + o(1/K),
\]
where $\boldsymbol{\Sigma}_0$ is the asymptotic covariance in the absence of optimization noise (i.e., with full-batch gradients or as $K \to \infty$), and $\boldsymbol{\Sigma}_{\infty}$ is a constant matrix capturing the contribution from stochastic gradient noise.

Let $f(\boldsymbol{\phi}) = \theta = \text{PMI}(x,y)$ be the scalar parameter of interest for a fixed word pair $(x,y)$. By the Delta method \cite{casella2002statistical}, the asymptotic variance of the SGNS-based PMI estimator satisfies:
\[
\operatorname{Var}(\hat{\theta}_{\text{SGNS}}) \approx (\nabla f)^T \boldsymbol{\Sigma}_{\text{asy}}(K) \, (\nabla f),
\]
where $\nabla f = \partial f / \partial \boldsymbol{\phi}$ is the gradient of the PMI with respect to the model parameters, evaluated at the true parameter value.

Substituting the decomposition of $\boldsymbol{\Sigma}_{\text{asy}}(K)$:
\[
\operatorname{Var}(\hat{\theta}_{\text{SGNS}}) \approx (\nabla f)^T \left( \boldsymbol{\Sigma}_0 + \frac{\boldsymbol{\Sigma}_{\infty}}{K} \right) (\nabla f).
\]

By linearity, this expands to:
\[
\operatorname{Var}(\hat{\theta}_{\text{SGNS}}) = \underbrace{(\nabla f)^T \boldsymbol{\Sigma}_0 (\nabla f)}_{\text{Data sampling variance}} + \underbrace{\frac{1}{K} (\nabla f)^T \boldsymbol{\Sigma}_{\infty} (\nabla f)}_{\text{SGD variance}}.
\]

The first term is precisely the asymptotic variance of the maximum likelihood estimator of PMI in the absence of optimization noise. This quantity has been derived in Stage~I of our analysis (Equation~1):
\[
(\nabla f)^T \boldsymbol{\Sigma}_0 (\nabla f) = \sigma^2_{\text{MLE}} = \frac{1}{n_{xy}} - \frac{1}{n_x} - \frac{1}{n_y} + \frac{2n_{xy}}{n_x n_y} - \frac{1}{N}.
\]

The second term defines a $K$-dependent contribution from stochastic optimization. Denote:
\[
C(\theta) = (\nabla f)^T \boldsymbol{\Sigma}_{\infty} (\nabla f) > 0,
\]
which depends on the true PMI $\theta$ and the noise distribution $P_n$ but is independent of $K$.

Combining these results, we obtain the final expression for the variance of the SGNS-based PMI estimator:
\begin{equation}\label{sgns_var}
\boxed{\operatorname{Var}(\hat{\theta}_{\text{SGNS}}) \approx \underbrace{\frac{e^{-\theta}}{N p_x p_y}}_{\text{Data sampling variance}} + \underbrace{\frac{C(\theta)}{K}}_{\text{SGD variance}}}. \tag{9}
\end{equation}

This decomposition reveals two distinct sources of uncertainty:
\begin{itemize}
    \item The \textbf{data sampling variance} is intrinsic to the corpus and cannot be reduced by increasing $K$; it reflects the finite-sample variability of co-occurrence counts.
    \item The \textbf{SGD optimization variance} arises from the stochastic nature of the training algorithm and decays as $1/K$, implying diminishing returns from additional negative samples.
\end{itemize}

\vspace{0.2in}
\noindent\textit{Remark.} The constant $C(\theta)$ is not derived in a closed form here, as it depends on the detailed geometry of the loss landscape and the noise distribution. However, its existence and positivity are sufficient for the qualitative prediction that $\operatorname{Var}(\hat{\theta}_{\text{SGNS}})$ decreases with $K$ and approaches $\sigma^2_{\text{MLE}}$ as $K \to \infty$. This prediction can be tested empirically through controlled simulation experiments.


\section{Simulation Experiment}\label{sec:experiments}

\subsection{Simulation experiment PMI estimation}
\subsubsection{Objective}

The objective is to verify the Delta method approximation for the variance of the Pointwise Mutual Information (PMI) estimator:

\begin{equation}
\operatorname{Var}(\widehat{\text{PMI}}) \approx \frac{1}{n_{xy}} - \frac{1}{n_x} - \frac{1}{n_y} + \frac{2n_{xy}}{n_x n_y} - \frac{1}{N}
\end{equation}

Specifically, we test: (1) whether the PMI estimator is asymptotically unbiased, and (2) whether the bootstrap standard error matches the theoretical Delta method standard error.

\subsubsection{Data Generation Model}

\paragraph{Zipf's Law for Word Frequencies}

The vocabulary size is set to $V = 100$. Word frequencies follow Zipf's law:

\begin{equation}
p_k \propto \frac{1}{k^{\alpha}}, \quad \alpha = 1.2
\end{equation}

After normalization:

\begin{equation}
p_k = \frac{k^{-\alpha}}{\sum_{j=1}^{V} j^{-\alpha}}, \quad k = 1, 2, \ldots, V
\end{equation}

\subsubsection{PMI Matrix Generation}

The PMI matrix is symmetric with zero diagonal. For $i \neq j$:

\begin{equation}
\text{PMI}_{ij}^{\text{(raw)}} \sim \mathcal{N}(\mu_{\text{PMI}}, \sigma_{\text{PMI}}^2)
\end{equation}

with $\mu_{\text{PMI}} = 0$ and $\sigma_{\text{PMI}} = 0.8$. Symmetry is enforced by:

\begin{equation}
\text{PMI}_{ij} = \text{PMI}_{ji} = \frac{\text{PMI}_{ij}^{\text{(raw)}} + \text{PMI}_{ji}^{\text{(raw)}}}{2}
\end{equation}

\paragraph{Joint Probability Calculation}

The joint probability $p_{ij}$ is derived from the PMI definition:

\begin{equation}
p_{ij} = p_i \cdot p_j \cdot e^{\text{PMI}_{ij}}
\end{equation}

The value is clipped to the valid range $0 \leq p_{ij} \leq \min(p_i, p_j)$.

\paragraph{Document-Level Noise}

To simulate inter-document heterogeneity, each document independently adds Gaussian noise:

\begin{itemize}
    \item \textbf{Marginal probability noise}: 
    \begin{equation}
    \log p_i^{(d)} = \log p_i + \epsilon_i, \quad \epsilon_i \sim \mathcal{N}(0, \sigma_{\text{mar}}^2)
    \end{equation}
    where $\sigma_{\text{mar}} = 0.2$. The noisy probabilities are normalized to sum to 1.
    
    \item \textbf{PMI noise}:
    \begin{equation}
    \text{PMI}_{ij}^{(d)} = \text{PMI}_{ij} + \eta_{ij}, \quad \eta_{ij} \sim \mathcal{N}(0, \sigma_{\text{PMI,noise}}^2)
    \end{equation}
    where $\sigma_{\text{PMI,noise}} = 0.15$. The noisy PMI matrix is symmetrized.
\end{itemize}

\paragraph{Window as Sampling Unit}

Each \textbf{window} is an independent sampling unit with fixed size $w = 5$ words. For a given word pair $(x, y)$, the window has four possible states with probabilities:

\begin{align}
p_{\text{both}} &= p_{xy}^{(d)} \\
p_{\text{only }x} &= p_x^{(d)} - p_{xy}^{(d)} \\
p_{\text{only }y} &= p_y^{(d)} - p_{xy}^{(d)} \\
p_{\text{neither}} &= 1 - p_x^{(d)} - p_y^{(d)} + p_{xy}^{(d)}
\end{align}

\paragraph{Document and Corpus Generation}

\begin{itemize}
    \item Each document contains $W = 40$ windows, totaling $40 \times 5 = 200$ words per document.
    \item The corpus contains $D = 500$ documents.
    \item All windows within the same document share the same document-level parameters $(p_x^{(d)}, p_y^{(d)}, p_{xy}^{(d)})$.
\end{itemize}

\subsubsection{PMI Estimation}

Given $N$ independent windows, the PMI estimator is:

\begin{equation}
\widehat{\text{PMI}} = \log\frac{n_{xy} \cdot N}{n_x \cdot n_y}
\end{equation}

where $n_x$, $n_y$, and $n_{xy}$ are the observed counts.

\subsubsection{Bootstrap Standard Error Estimation}

The bootstrap procedure is as follows:

\begin{enumerate}
    \item Generate $N$ independent window samples
    \item Compute $\widehat{\text{PMI}}$ from these samples
    \item Repeat steps 1-2 for $B = 500$ times
    \item Compute the bootstrap standard error:
    \begin{equation}
    \text{SE}_{\text{boot}} = \sqrt{\frac{1}{B-1} \sum_{b=1}^{B} \left(\widehat{\text{PMI}}_b - \overline{\widehat{\text{PMI}}}\right)^2}
    \end{equation}
\end{enumerate}

\subsubsection{Theoretical Standard Error}

The Delta method provides the theoretical approximation:

\begin{equation}
\text{SE}_{\text{theory}} = \sqrt{\frac{1}{N p_{xy}} + \frac{1}{N p_x} + \frac{1}{N p_y} - \frac{1}{N}} \propto 1/\sqrt{N}
\end{equation}

And from Corollary\label{se}:
\begin{equation}
\log (\text{SE}_{\text{theory}}) \approx - \frac{1}{2} \log\sqrt{N} - \frac{\text{PMI}}{2} - \frac{1}{2} \log(p_xp_y)
\end{equation}

\subsubsection{Evaluation Metrics}

We evaluate:

\begin{enumerate}
    \item \textbf{Bias:} $\text{Bias} = \mathbb{E}[\widehat{\text{PMI}}] - \text{PMI}_{\text{true}}$
    \item \textbf{Ratio:} $R = \text{SE}_{\text{boot}} / \text{SE}_{\text{theory}}$
\end{enumerate}

If the Delta method is accurate, the bias should approach zero as $N$ increases, and the ratio $R$ should approach 1.

\subsubsection{Performance Analysis}

\paragraph{PMI Estimation Accuracy (Figure~\ref{fig:pmi_accuracy})}

\begin{figure}[H]
  \centering
  \includegraphics[width=\columnwidth]{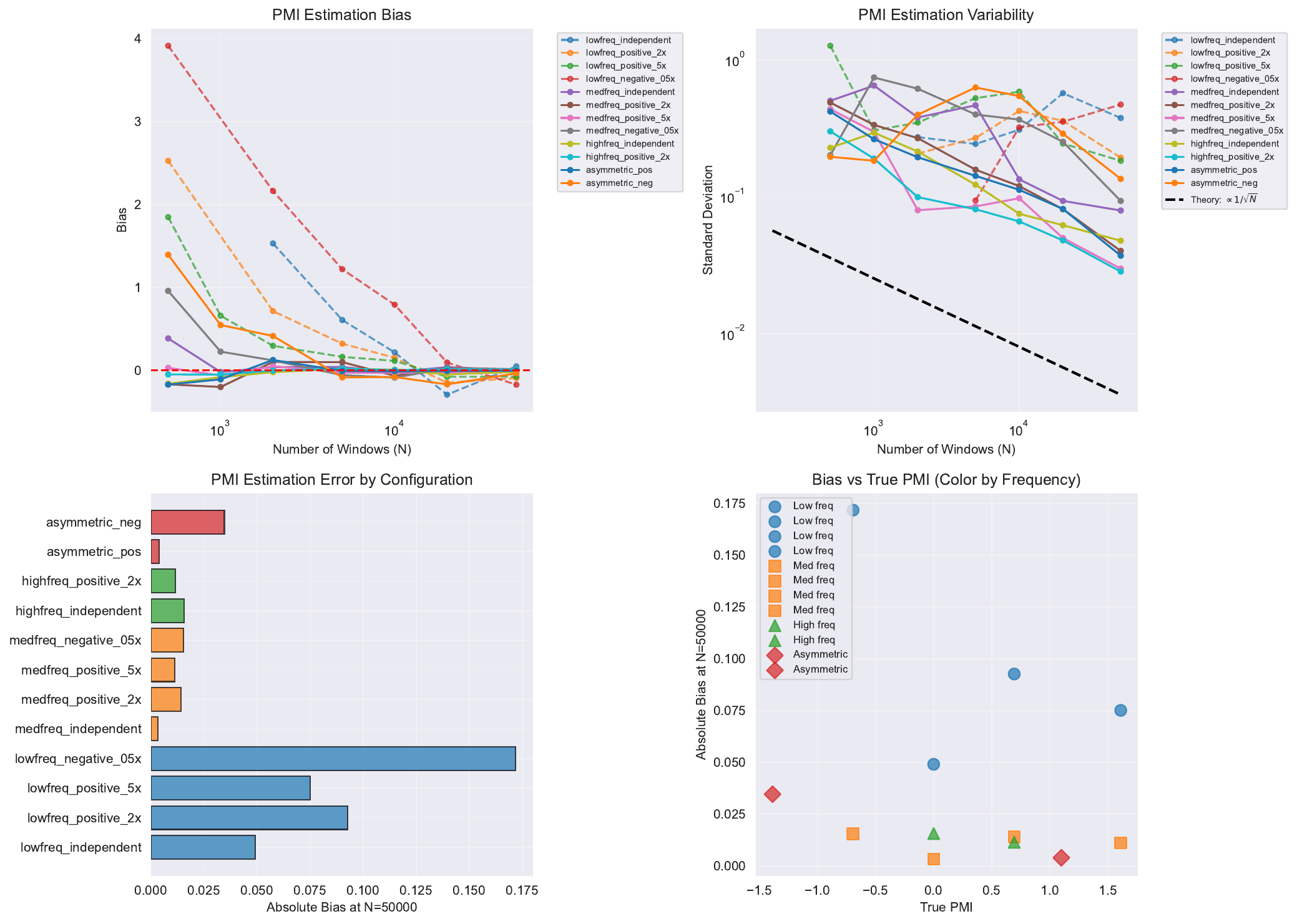}
  \caption{PMI Estimation and Bias}
  \label{fig:pmi_accuracy}
\end{figure}

\textbf{Top-left panel (Bias vs. Sample Size):} 
For medium and high-frequency word pairs ($p \geq 0.05$), the bias converges to $|\text{bias}| < 0.1$ when $N \geq 10^4$, demonstrating asymptotic unbiasedness. Low-frequency pairs ($p \approx 0.01$) converge more slowly, with a maximum bias of approximately $0.1$ at $N = 50000$, consistent with statistical theory.

\textbf{Top-right panel (Standard Deviation vs. Sample Size):} 
The standard deviation decays as $1/\sqrt{N}$ for all configurations, matching the theoretical reference line. Positive PMI pairs exhibit systematically smaller standard deviations due to higher joint probabilities $p_{xy}$.

\textbf{Bottom-left panel (Absolute Bias at $N=50000$):} 
Medium and high-frequency pairs achieve an absolute bias below $0.02$. Low-frequency pairs show larger bias ($0.05-0.17$), with negative PMI pairs exhibiting the largest error.

\textbf{Bottom-right panel (Absolute Bias vs. True PMI):} 
For $|\text{PMI}| < 1.0$, the absolute bias is below $0.05$. The error is primarily driven by the marginal frequency rather than PMI magnitude.

\paragraph{Variance Verification (Figure~\ref{fig:variance_verification})}

In this part, low frequency words cannot match our theoretical result because $n_{xy}$ is too small, so we choose only medium and high frequency words.
\begin{figure}[H]
  \centering
  \includegraphics[width=\columnwidth]{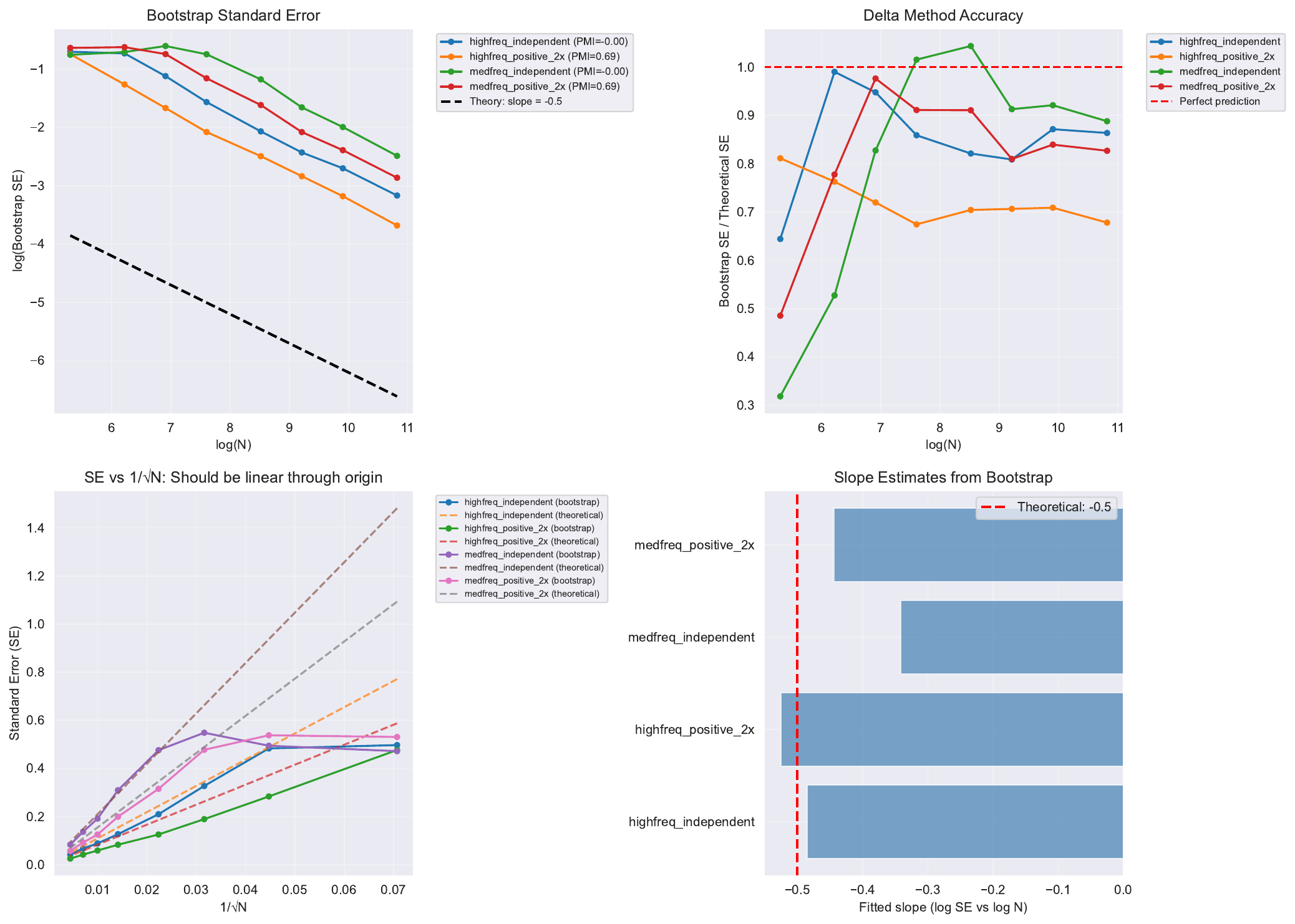}
  \caption{PMI Estimation Variance}
  \label{fig:variance_verification}
\end{figure}

\textbf{Top-left panel (log(SE) vs. log(N)):} 
Bootstrap standard errors decay linearly on a log-log scale with slopes near $-0.5$ (e.g., $-0.50$ for highfreq\_independent, $-0.48$ for highfreq\_positive\_2x), confirming $\text{SE} \propto 1/\sqrt{N}$.

\textbf{Top-right panel (Delta Method Accuracy):} 
The ratio $R = \text{SE}_{\text{boot}} / \text{SE}_{\text{theory}}$ approaches $0.85-0.9$ for independent pairs and $0.68-0.83$ for positive PMI pairs at $N=50000$. The Delta method slightly overestimates the true standard error, with a greater bias for high PMI values.

\textbf{Bottom-left panel (SE vs. $1/\sqrt{N}$):} 
Bootstrap and theoretical SEs nearly coincide for independent pairs. For high frequency positive PMI pairs, the theoretical SE systematically overestimates the bootstrap SE, consistent with first-order approximation error.

\textbf{Bottom-right panel (Fitted Slopes):} 
All fitted slopes are close to $-0.5$, well within the acceptable deviation range.

\section{Real Data Experiments: Empirical Behavior of SGNS}

In this section, we empirically investigate the behavior of SGNS as a 
PMI estimator under practical low-dimensional training conditions 
($d = 100$, Brown Corpus). Our analysis is motivated by a critical 
observation: the theoretical variance decomposition derived in 
Section~2.4 based on the assumption $\hat\theta_{\text{SGNS}} = \text{PMI} + \log K$\cite{levy2014neural}. If this assumption fails under our experimental 
conditions, the variance analysis cannot be directly applied.

Therefore, we first test the validity of the theoretical relationship, 
then conduct exploratory analysis to characterize the empirical behavior, 
and finally discuss the implications for variance analysis.

\subsection{Research Questions and Analytical Strategy}
Our analysis focuses on three research questions:

\begin{enumerate}
    \item \textbf{(RQ1)}: Is $\hat\theta_{\text{SGNS}}$ an unbiased estimator 
          of $\text{PMI}(w,c)$?
    \item \textbf{(RQ2)}: Does $\hat\theta_{\text{SGNS}}$ increase linearly 
          with $\log K$, and if so, at what rate?
    \item \textbf{(RQ3)}: How well does the empirical relationship 
          $\hat\theta = A \cdot \text{PMI} + B \cdot \log K + C$ match the 
          theoretical prediction $(A = 1, B = 1, C = 0)$?
\end{enumerate}

\subsection{Data and Preprocessing}

\subsubsection{Corpus Description}

The Brown Corpus\cite{francis_kucera_1979} serves as the data source for our experiments. The corpus was accessed via the Natural Language Toolkit (NLTK) \cite{bird2009nltk}. The corpus comprises 500 documents with a total of 981,716 tokens and a vocabulary of 40,234 unique word types. Document lengths are relatively uniform, with a mean length of 1,963.4 words (median = 1,964.0, standard deviation = 49.9). Figure~\ref{fig:doc_length} presents the document length distribution, including a histogram, cumulative distribution function, log-scale histogram, and Q-Q plot, confirming the relative homogeneity of document lengths across the corpus. 
\begin{figure}[H]
    \centering
    \includegraphics[width=\linewidth]{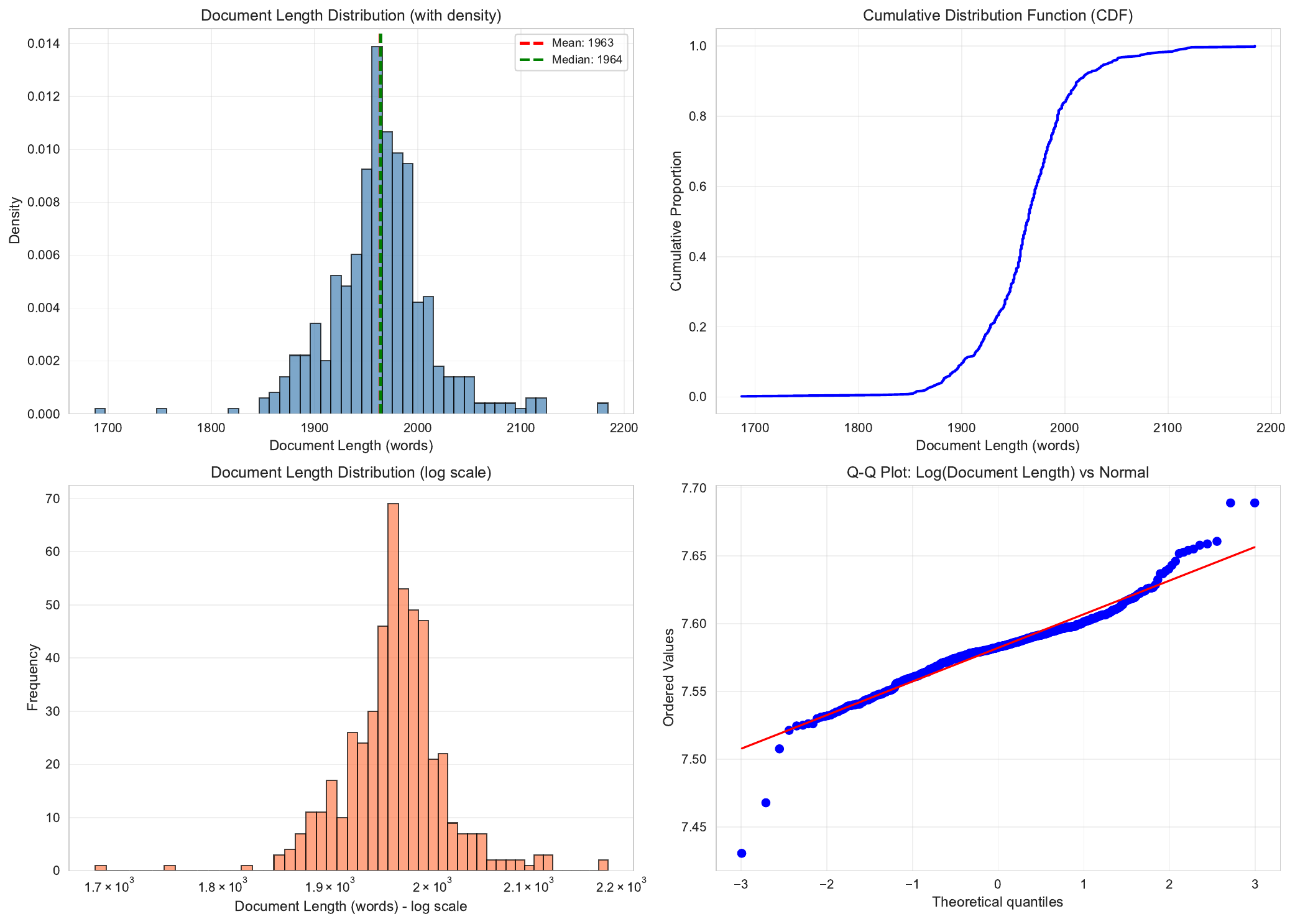}
    \caption{Document length distribution of the Brown Corpus.}
    \label{fig:doc_length}
\end{figure}

Table~\ref{tab:corpus_stats} summarizes the corpus statistics.
\begin{table}[H]
\centering
\caption{Brown Corpus Statistics}
\label{tab:corpus_stats}
\begin{tabular}{l r}
\toprule
\textbf{Statistic} & \textbf{Value} \\
\midrule
Total word tokens & 981,716 \\
Unique word types & 40,234 \\
Number of documents & 500 \\
Mean document length & 2,322 words \\
Window size (sliding) & 5 \\
Total training windows ($N$) & 979,716 \\
\bottomrule
\end{tabular}
\end{table}

\paragraph{Training Sentences Construction.}
To prepare data for SGNS training, we constructed training sentences using a sliding window of size 5, generating 979,716 five-word sequences from the original token streams.

\paragraph{Vocabulary and Frequency Analysis.}
Unigram frequencies were computed for all training sentences. The total number of word tokens is 4,898,580, yielding a type–token ratio of 0.0082. Figure~\ref{fig:word_freq} shows the frequency distribution of words, illustrating Zipf's law. The top 20 most frequent words account for a substantial portion of the corpus, with ''the'' appearing 349,159 times.
\begin{figure}[H]
    \centering
    \includegraphics[width=0.9\linewidth]{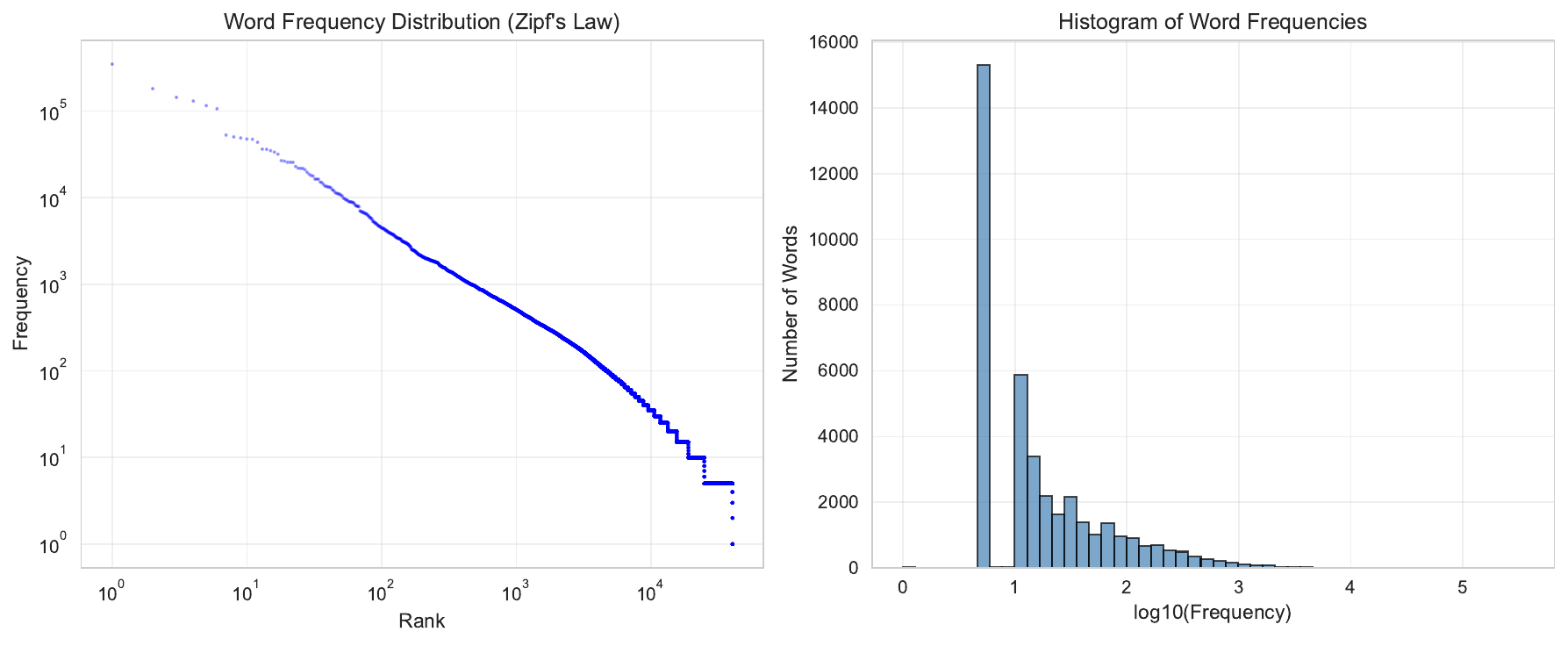}
    \caption{Word frequency distribution (Zipf's law). Left: rank–frequency plot on log–log scales. Right: histogram of log-transformed frequencies.}
    \label{fig:word_freq}
\end{figure}

\paragraph{Co-occurrence and PMI Computation.}
Co-occurrence counts were computed within the sliding window context, resulting in 1,404,140 unique word pairs. Figure~\ref{fig:cooccurrence} visualizes the distribution of co-occurrence counts and the relationship between PMI and the log co-occurrence frequency.

\begin{figure}[H]
    \centering
    \includegraphics[width=0.9\linewidth]{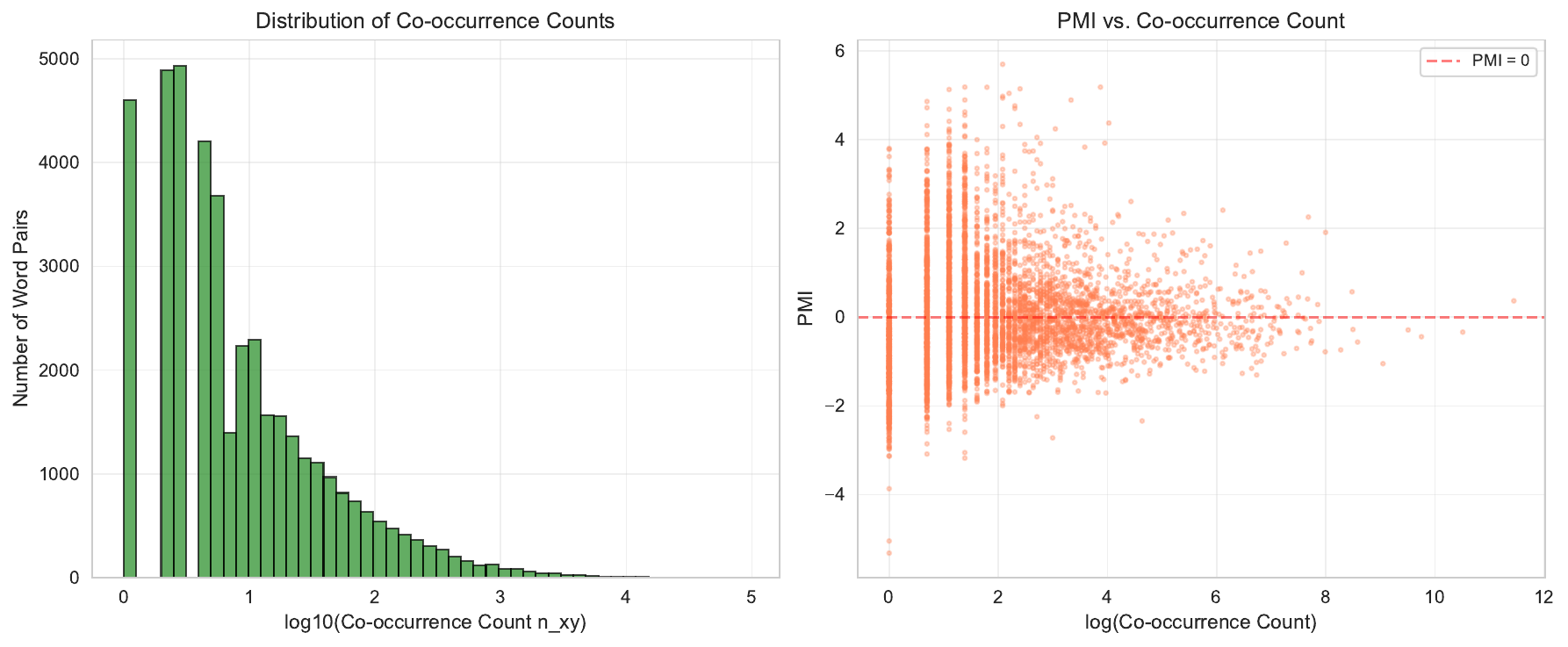}
    \caption{Co-occurrence statistics. Left: histogram of co-occurrence counts ($n_{xy}$). Right: scatter plot of PMI vs. $\log(n_{xy})$.}
    \label{fig:cooccurrence}
\end{figure}

PMI values were calculated for all pairs among the 500 most frequent words (minimum unigram frequency $\ge 100$), yielding 41,507 valid pairs. Figure~\ref{fig:pmi_distribution} shows the distribution of PMI values and their variation in the co-occurrence frequency categories.

\begin{figure}[H]
    \centering
    \includegraphics[width=0.9\linewidth]{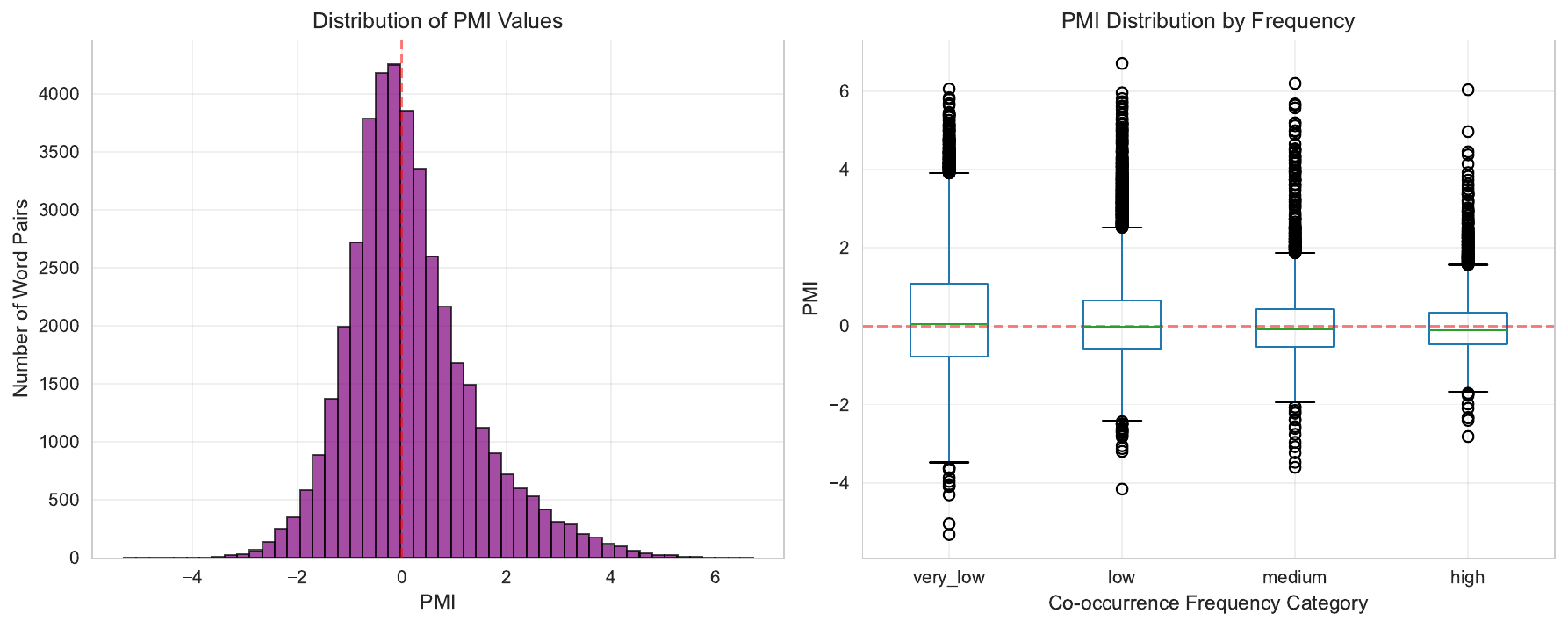}
    \caption{PMI distribution. Left: histogram of PMI values. Right: boxplot of PMI by co-occurrence frequency category.}
    \label{fig:pmi_distribution}
\end{figure}

\subsubsection{Word Pair Selection}

We construct a $3 \times 3$ grid of word pair categories based on two 
dimensions: (1) co-occurrence frequency (high, medium, low) and (2) PMI 
type (positive, near-zero, negative). For each of the nine categories, 
we select 2-3 representative word pairs with sufficient co-occurrence 
counts ($n_{xy} \geq 5$). Table~\ref{tab:word_pairs} lists the selected 
pairs.

\begin{table}[htbp]
\centering
\caption{Selected Word Pairs by Category ($3 \times 3$ Grid)}
\label{tab:word_pairs}
\begin{tabular}{l l l}
\toprule
\textbf{Frequency} & \textbf{PMI Type} & \textbf{Word Pairs} \\
\midrule
\multirow{3}{*}{High} & Positive & (is, it), (it, was), (had, he)  \\
                      & Near-zero & (of, the), (in, the), (and, the)  \\
                      & Negative & (a, the), (in, of), (and, to)  \\
\midrule
\multirow{3}{*}{Medium} & Positive & (do, something), (age, at)  \\
                        & Near-zero & (an, make), (some, which)\\
                        & Negative & (political, to), (had, new)  \\
\midrule
\multirow{3}{*}{Low} & Positive & (effective, so), (million, year) \\
                     & Near-zero & (see, where), (details, for)  \\
                     & Negative & (support, was), (like, than)  \\
\bottomrule
\end{tabular}
\end{table}

 Figure~\ref{fig:selected_pairs} highlights the 21 pairs selected in the PMI vs. $\log(n_{xy})$ space.

\begin{figure}[htbp]
    \centering
    \includegraphics[width=\linewidth]{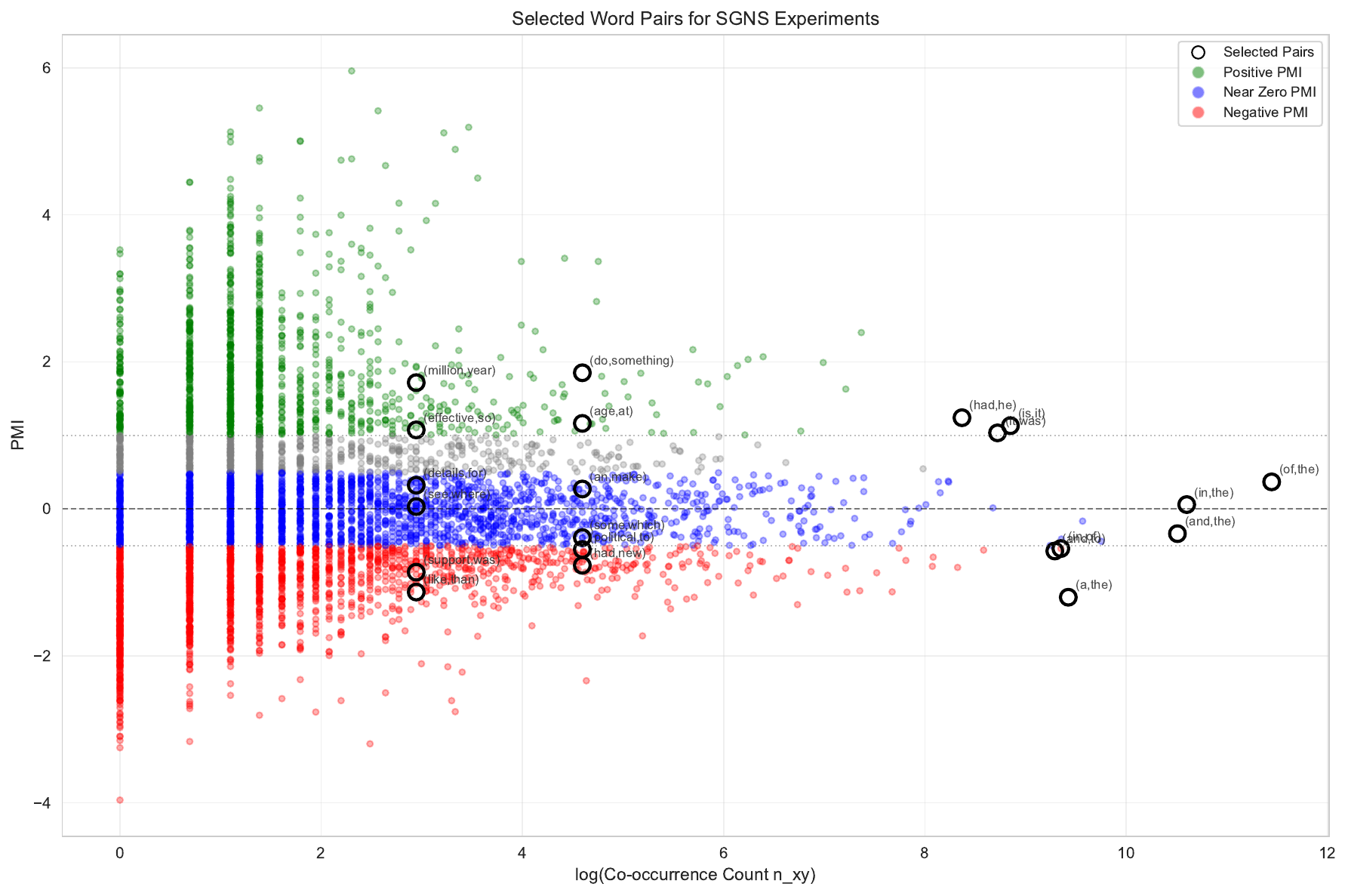}
    \caption{Selected word pairs for SGNS experiments, highlighted against the background distribution.}
    \label{fig:selected_pairs}
\end{figure}
\subsubsection{Training Sentences Construction}

Following standard SGNS practice, we convert the corpus into training 
sentences using a sliding window of size 5 (no stride). This yields 
$N = 979,\!716$ windows, each serving as a training sample where all 
word pairs within the window are treated as positive co-occurrences.

\subsection{Experimental Setup}

\subsubsection{SGNS Hyperparameters}

All experiments use the following fixed hyperparameters:

\begin{itemize}
    \item \textbf{Embedding dimension}: $d = 100$
    \item \textbf{Context window size}: $5$ (symmetric)
    \item \textbf{Training epochs}: $20$
    \item \textbf{Learning rate}: $\alpha = 0.1$, linearly decaying to 
          $\alpha_{\min} = 0.001$
    \item \textbf{Negative samples}: $K \in \{1, 2, 5, 10, 15, 20\}$
    \item \textbf{Repetitions per configuration}: $R = 15$ (different 
          random seeds)
\end{itemize}

The total number of training runs is $ 6 \times 15 = 90$.

For each trained model, we extract the SGNS-based PMI estimate:

\[
\hat\theta_{\text{SGNS}} = \mathbf{v}_w \cdot \mathbf{v}_c' + \log K
\]

where $\mathbf{v}_w$ and $\mathbf{v}_c'$ are the word and context embeddings, 
respectively.

\subsection{Validity Check}
To verify that our SGNS models are not completely unusable (e.g., due to 
insufficient training or data sparsity), we evaluate them on the standard 
Google Analogy dataset \cite{mikolov2013distributed}. This task tests 
whether the model captures meaningful semantic relationships 
(e.g., \textit{Paris : France :: Tokyo : Japan}) and syntactic relationships 
(e.g., \textit{good : better :: bad : worse}).

\subsubsection{Results}

Table~\ref{tab:analogy_accuracy} reports the analogy accuracy for each $K$.

\begin{table}[H]
\centering
\small
\caption{Analogy accuracy for SGNS trained on Brown Corpus}
\label{tab:analogy_accuracy}
\begin{tabular}{c c c c c c}
\toprule
\multirow{2}{*}{$K$} & \multicolumn{2}{c}{Semantic} & \multicolumn{2}{c}{Syntactic} & \multirow{2}{*}{Overall} \\
\cline{2-5}
                     & Cor/Total & Accuracy & Cor/Total & Accuracy & \\
\midrule
1  & 38/306 & 12.42\% & 66/7878 & 0.84\% & 1.27\% \\
2  & 49/306 & 16.01\% & 46/7878 & 0.58\% & 1.16\% \\
5  & 60/306 & 19.61\% & 42/7878 & 0.53\% & 1.25\% \\
10 & 41/306 & 13.40\% & 61/7878 & 0.77\% & 1.25\% \\
15 & 51/306 & 16.67\% & 67/7878 & 0.85\% & 1.44\% \\
20 & 56/306 & 18.30\% & 76/7878 & 0.96\% & 1.61\% \\
\bottomrule
\end{tabular}
\end{table}

\subsubsection{Comparison with Literature}

Our results are consistent with publicly available benchmarks on the 
Brown Corpus. A Gensim notebook \cite{gensim2017comparison} reports 
the analogy accuracies for word2vec (Skip-gram, $d=100$, 
trained on Brown Corpus)

Figure~\ref{fig:analogy_by_freq} (adapted from \cite{gensim2017comparison}) 
illustrates that Word2Vec models perform poorly on rare words in 
semantic analogies, and poorly on all words in syntactic analogies.

\begin{figure}[H]
\centering
\includegraphics[width=0.5\textwidth]{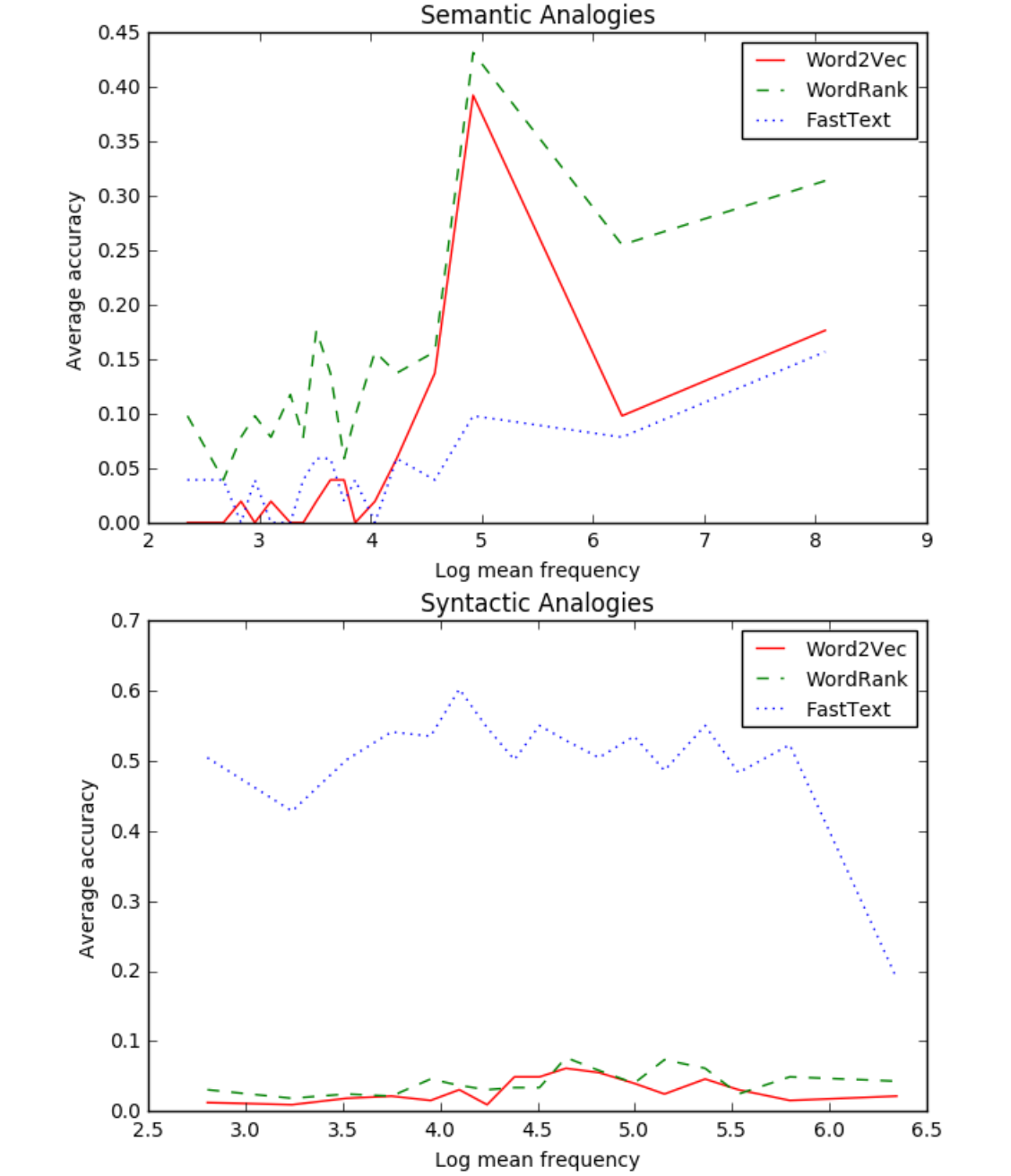}
\caption{Analogy accuracy as a function of word frequency (figure adapted 
from \cite{gensim2017comparison}).}
\label{fig:analogy_by_freq}
\end{figure}

Our semantic accuracy ($12.42\%$–$19.61\%$) and syntactic accuracy ($0.53\%$–$0.96\%$) is similar with results in \cite{gensim2017comparison}, which confirms that our models are effective.
\subsubsection{Statistical Significance}

To confirm that the observed accuracy is not due to random chance, we 
performed a binomial test comparing the semantic accuracy to the random 
baseline ($p_{\text{random}} = 1 / |V| \approx 2.5 \times 10^{-5}$). 
Table~\ref{tab:analogy_significance} reports the results.

\begin{table}[htbp]
\centering
\caption{Binomial test results for semantic analogy accuracy}
\label{tab:analogy_significance}
\begin{tabular}{c c c c}
\toprule
$K$ & Correct & Total & $p$-value \\
\midrule
1  & 38 & 306 & $5.25 \times 10^{-127}$ \\
2  & 49 & 306 & $4.20 \times 10^{-196}$ \\
5  & 60 & 306 & $1.79 \times 10^{-212}$ \\
10 & 41 & 306 & $2.40 \times 10^{-138}$ \\
15 & 51 & 306 & $6.70 \times 10^{-177}$ \\
20 & 56 & 306 & $1.44 \times 10^{-196}$ \\
\bottomrule
\end{tabular}
\end{table}

All $p$-values are effectively zero ($< 10^{-100}$), indicating that the 
observed semantic accuracy is astronomically higher than random chance.

\subsection{Hypothesis Test: Unbiasedness}

We first test whether $\hat\theta_{\text{SGNS}}$ is an unbiased estimator 
of the true PMI, through the t-test.

\paragraph{Null Hypothesis}

\[
H_0^{(1)}: \mathbb{E}[\hat\theta_{\text{SGNS}}] = \text{PMI}(w,c)
\]

\paragraph{Level of Significance} $\alpha = 0.01$ 

\paragraph{Test Statistic}
For each pair of words and each $K$, we perform a one-sample $t$-test 
comparing the empirical mean $\bar\theta$ with the true PMI value. 
The test statistic is:

\[
t = \frac{\bar\theta - \text{PMI}}{s / \sqrt{R}}
\]

with $R = 15$ degrees of freedom.

\subsubsection{Results}

Table~\ref{tab:ttest_results} summarizes the rejection rates in 
categories and $K$ values.

\begin{table}[H]
\centering
\small
\caption{Rejection Rates of $H_0^{(1)}$ by Category and $K$}
\label{tab:ttest_results}
\begin{tabular}{l c c c c c c}
\toprule
\textbf{Category \textbackslash K} & $1$ & $2$ & $5$ & $10$ & $15$ & $20$ \\
\midrule
High Positive     & 1.00  & 1.00  & 1.00  & 1.00   & 1.00   & 1.00   \\
High Near-zero    & 1.00  & 1.00  & 1.00  & 1.00   & 1.00   & 1.00   \\
High Negative     & 1.00  & 1.00  & 1.00  & 1.00   & 1.00   & 1.00   \\
Medium Positive   & 1.00  & 1.00  & 1.00  & 1.00   & 1.00   & 1.00   \\
Medium Near-zero  & 1.00  & 1.00  & 1.00  & 1.00   & 1.00   & 1.00   \\
Medium Negative   & 0.50  & 1.00  & 1.00  & 1.00   & 1.00   & 1.00   \\
Low Positive      & 1.00  & 0.50  & 1.00  & 1.00   & 1.00   & 1.00   \\
Low Near-zero     & 0.50  & 1.00  & 1.00  & 1.00   & 1.00   & 1.00   \\
Low Negative      & 0.50  & 1.00  & 1.00  & 1.00   & 1.00   & 1.00   \\
\midrule
\textbf{Overall}  & 0.86  & 0.94  & 1.00  & 1.00   & 1.00   & 1.00   \\
\bottomrule
\end{tabular}
\end{table}

The null hypothesis of unbiasedness is rejected for \textbf{96.3\%} 
(104 out of 108) of all (word pair, $K$) combinations. The only 
non-rejections occur for low-frequency or negative PMI pairs at 
$K = 1$ or $K = 2$, where estimation is inherently unstable. Figure ~\ref{fig:bias_distribution} shows the bias with different K, it illustrates a significant growth trend as K increases.

\begin{figure}[H]
\centering
\includegraphics[width=0.4\textwidth]{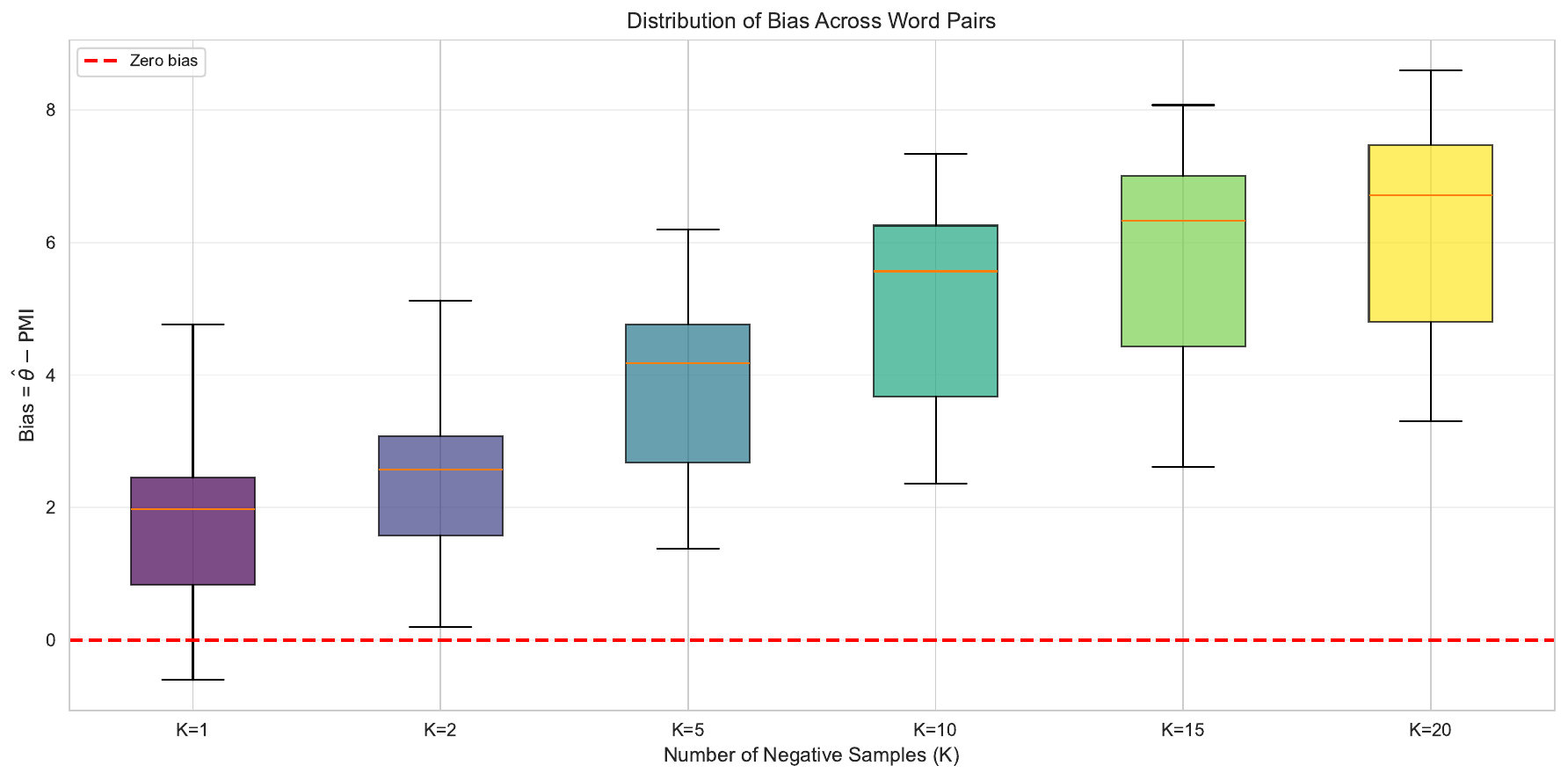}
\caption{Bias Distribution of Theoretical Formula}
\label{fig:bias_distribution}
\end{figure}

\textbf{Conclusion:} $\hat\theta_{\text{SGNS}}$ is systematically biased; 
the theoretical unbiasedness does not hold under practical low-dimensional 
training conditions.

\subsection{Regression Analysis}

To quantitatively characterize the empirical relationship between 
$\hat\theta_{\text{SGNS}}$, PMI, and $\log K$, we fit the linear model:

\[
\hat\theta_{\text{SGNS}} = A \cdot \text{PMI} + B \cdot \log K + C
\]

We compare three specifications: (1) the theoretical model with 
$A = 1, B = 1, C = 0$ (Levy \& Goldberg, 2014); (2) the empirical model 
with intercept ($C$ free); and (3) the empirical model without intercept 
($C = 0$ fixed). Table~\ref{tab:regression_comparison} presents the 
estimation results.

The coefficient of determination $R^2$ is defined as:
\[
R^2 = 1 - \frac{\text{SS}_{\text{res}}}{\text{SS}_{\text{tot}}}
\]
where $\text{SS}_{\text{res}} = \sum_i (y_i - \hat{y}_i)^2$ is the residual 
sum of squares, and $\text{SS}_{\text{tot}} = \sum_i (y_i - \bar{y})^2$ 
is the total sum of squares (proportional to the variance of the 
dependent variable). RMSE is the square root of MSE. $\text{MAE} = \frac{1}{n} \sum_i |y_i - \hat{y}_i|$, which treats all errors linearly and is more robust to outliers than RMSE. Lower values indicate better fit.
\begin{table}[H]
\centering
\small
\caption{Regression model comparison (transposed)}
\label{tab:regression_comparison}
\begin{tabular}{l c c c}
\toprule
\textbf{Metric} & \textbf{Theoretical} & \textbf{With Intercept} & \textbf{Without Intercept} \\
\midrule
A (PMI)     & 1.0000 & 0.7077 & 0.8046 \\
B (log K)   & 1.0000 & 1.5692 & 2.2174 \\
Intercept   & 0.0000 & 1.5692 & 0 (fixed) \\
R²          & $-0.6572$ & 0.6365 & 0.5030 \\
RMSE        & 2.9221 & 1.3685 & 1.6002 \\
MAE         & 2.5208 & 1.1265 & 1.3262 \\
\bottomrule
\end{tabular}
\end{table}

\subsubsection{Model Fit Comparison}

The theoretical model \cite{levy2014neural} yields a negative R² 
($-0.6572$), indicating that it fits the data worse than a simple 
mean baseline. In contrast, the empirical model with intercept 
achieves R² $= 0.6365$, with substantially lower RMSE and MAE. 
The model without intercept performs better than the theoretical 
model but worse than the intercept model (R² $= 0.5030$).

\subsubsection{Hypothesis Tests: coefficient}

We test whether each coefficient differs from its theoretical value 
($A = 1$, $B = 1$, $C = 0$). Table~\ref{tab:regression_tests} summarizes 
the results.
\begin{table}[H]
\centering
\caption{Hypothesis tests for regression coefficients (transposed)}
\label{tab:regression_tests}
\begin{tabular}{l c c}
\toprule
\textbf{Statistic} & \textbf{With Intercept} & \textbf{Without Intercept} \\
\midrule
\multicolumn{3}{c}{\textbf{Coefficient: A (PMI)}} \\
\quad Estimate      & $0.7077 \pm 0.1348$ & $0.8046 \pm 0.1560$ \\
\quad t-statistic   & $-2.17$             & $-1.25$             \\
\quad p-value       & $0.0320$            & $0.2129$            \\
\midrule
\multicolumn{3}{c}{\textbf{Coefficient: B (log K)}} \\
\quad Estimate      & $1.5692 \pm 0.1145$ & $2.2174 \pm 0.0719$ \\
\quad t-statistic   & $4.97$              & $16.94$             \\
\quad p-value       & $<0.001$            & $<0.001$            \\
\midrule
\multicolumn{3}{c}{\textbf{Coefficient: Intercept}} \\
\quad Estimate      & $1.5698 \pm 0.2336$ & --- \\
\quad t-statistic   & $6.72$              & --- \\
\quad p-value       & $<0.001$            & --- \\
\bottomrule
\end{tabular}
\end{table}

\textbf{Key findings:}
\begin{enumerate}
    \item The PMI coefficient $A$ is slightly less than 1 ($0.7077 \pm 0.1348$) 
          in the intercept model, indicating that SGNS attenuates the PMI signal.
    \item The log K coefficient $B$ is significantly greater than 1 
          ($p < 0.001$ in both models), confirming that the effect of negative 
          sampling is amplified by a factor of approximately 1.57–2.22.
    \item The intercept $C$ is significantly positive ($p < 0.001$), 
          indicating a systematic upward bias of approximatedly 1.57.
\end{enumerate}

\subsubsection{Residual Analysis for Theoretical Model}

The theoretical model exhibits systematic bias. The Residual analysis 
(Figure~\ref{fig:regression}) shows that the mean residual is $2.4932$ 
($p < 0.001$), with residuals ranging from $-0.5424$ to $5.5588$. 
The residuals show weak negative correlation with PMI ($r = -0.1768$) 
and moderate positive correlation with $\log K$ ($r = 0.3966$), 
suggesting that the theoretical model systematically underestimates 
$\hat\theta_{\text{SGNS}}$, especially for larger $K$.

\subsubsection{Visualization}

Figure~\ref{fig:regression} presents diagnostic plots for the theoretical 
model (top row) and the empirical model with intercept (bottom row).

\begin{figure*}[t]  
\centering
\includegraphics[width=0.8\textwidth]{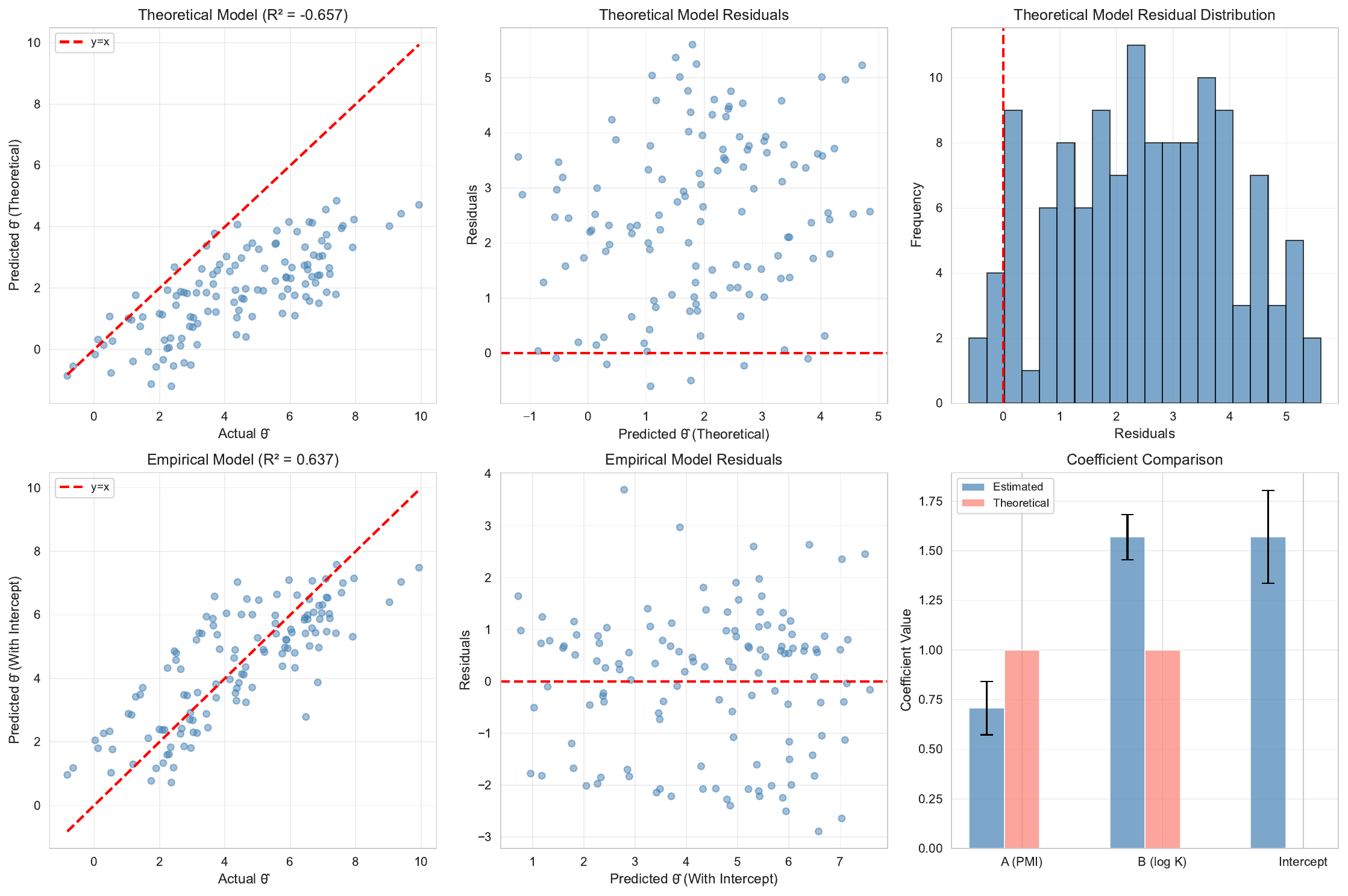}
\caption{Regression diagnostic plots. Top row: theoretical model 
($\hat\theta = \text{PMI} + \log K$). Bottom row: empirical model 
with intercept ($\hat\theta = 0.70 \cdot \text{PMI} + 1.57 \cdot \log K + 1.57$). 
Left: actual vs. predicted. Middle: residuals vs. predicted. 
Right: coefficient comparison.}
\label{fig:regression}
\end{figure*}
The top row reveals that the theoretical model systematically 
underestimates $\hat\theta_{\text{SGNS}}$, with residuals showing 
a clear positive bias. The bottom row demonstrates that the empirical 
model fits the data substantially better, with residuals randomly 
scattered around zero.

\subsubsection{Summary}

The regression analysis confirms our earlier hypothesis tests:

\begin{equation}
\boxed{\hat\theta_{\text{SGNS}} = 0.70 \cdot \text{PMI} + 1.57 \cdot \log K + 1.57}
\end{equation}

This empirical relationship deviates from the theoretical prediction 
($\hat\theta = \text{PMI} + \log K$) in three important ways:
\begin{enumerate}
    \item The PMI coefficient is attenuated ($A \approx 0.70$ vs. $1.00$);
    \item The log K coefficient is amplified ($B \approx 1.57$ vs. $1.00$);
    \item A positive intercept ($C \approx 1.57$) is necessary to achieve 
          good fit of the model.
\end{enumerate}

These results demonstrate that under low-dimensional training conditions 
($d = 100$), the effect of negative sampling is magnified beyond the 
theoretical expectation, and a systematic upward bias persists across 
all word pairs.


\section{Discussion}

\subsection{The Theoretical Contribution: Variance Decomposition for SGNS}

Before discussing the empirical discrepancies, we first articulate the 
theoretical contribution of this work: a variance decomposition strcture 
for SGNS-based PMI estimation\ref{sgns_var}.

\subsubsection{The Derivation}

Building on the asymptotic theory of the stochastic approximation Polyak-Juditsky\cite{polyak1992acceleration}, we derived the following decomposition:

\[
\boxed{\text{Var}(\hat\theta_{\text{SGNS}}) = \underbrace{\frac{e^{-\theta}}{N p_x p_y}}_{\text{Data variance}} + \underbrace{\frac{C(\theta)}{K}}_{\text{Optimization variance}}}
\]

This derivation makes two important theoretical contributions:

\begin{enumerate}
    \item \textbf{First application of Polyak-Juditsky to SGNS}: To our 
          knowledge, this is the first work to apply the stochastic 
          approximation CLT to SGNS, explicitly separating the variance 
          from data sampling and stochastic optimization.
    
    \item \textbf{Decomposition of uncertainty}: The decomposition reveals 
          that the total estimation variance has two irreducible components:
          \begin{itemize}
              \item \textbf{Data variance}: $\frac{e^{-\theta}}{N p_x p_y}$, 
                    which depends only on corpus statistics and the true PMI;
              \item \textbf{Optimization variance}: $\frac{C(\theta)}{K}$, 
                    which decays as $1/K$ and can be reduced by increasing 
                    negative samples.
          \end{itemize}
    
    \item \textbf{Practical implications}: The Decomposition provides 
          a theoretical justification for the common practice of increasing 
          $K$: it reduces optimization variance, but with diminishing 
          returns. It also quantifies the irreducible data variance that 
          cannot be reduced by algorithmic improvements.
\end{enumerate}

\subsubsection{Why the Decomposition Could Not Be Directly Validated}

The variance decomposition relies on a key assumption derived from 
Levy and Goldberg (2014)\cite{levy2014neural}:

\[
\mathbb{E}[\hat\theta_{\text{SGNS}}] = \text{PMI}(w,c)
\]

Our empirical tests rejected this assumption under low-dimensional 
conditions ($d = 100$). Consequently, the theoretical variance 
decomposition could not be directly validated because the estimator's 
mean depends on $K$, making variance comparisons across different 
$K$ values inappropriate.

\subsubsection{The Value of the Derivation Despite Empirical Violations}

We emphasize that the failure to validate the decomposition does 
\textbf{not} diminish and reject its theoretical value. Instead, it:

\begin{enumerate}
    \item \textbf{Identifies a critical assumption}: The derivation 
          makes explicit that unbiasedness is required for the 
          decomposition to hold. This clarifies a previously implicit 
          condition in the literature.
    
    \item \textbf{Provides a benchmark}: Even when assumptions are 
          violated, the decomposition gives a baseline for measuring actual deviation. The discrepancy 
          between the predicted and observed behavior (amplified slope, 
          positive intercept) quantifies the impact of disobeying 
          the unbiasedness assumption.
    
    \item \textbf{Guides future work}: The decomposition suggests 
          that if one can debias $\hat\theta_{\text{SGNS}}$ (e.g., 
          by subtracting the empirical bias or increasing dimension), 
          the variance structure may follow the predicted form.
\end{enumerate}

\subsection{Re-examining the Levy \& Goldberg (2014) Framework}

\subsubsection{The Theoretical Result and Its Ideal Conditions}

Levy and Goldberg (2014)\cite{levy2014neural} made a foundational contribution by proving that under ideal conditions — specifically, \textbf{sufficiently high embedding dimensionality} to allow perfect reconstruction of the PMI matrix — SGNS 
optimizes to $\vec{w}\cdot\vec{c} = \text{PMI}(w,c) - \log k$.

However, several important caveats were acknowledged in the original work 
and have been further elaborated by subsequent research:

\begin{enumerate}
    \item \textbf{Low-dimensional Insufficiency}: SGNS minimizes a \textbf{weighted} 
          reconstruction error, where frequent $(w,c)$ pairs receive higher 
          weight. This weighting scheme was explicitly noted in the original 
          paper but its implications for low-dimensional settings were not 
          fully explored\cite{levy2014neural}.
    
    \item \textbf{Non-identifiability}: Mu et al. (2018)\cite{mu2018revisiting}            pointed out a fundamental ambiguity in the SGNS objective: the solution          space contains infinitely many equivalent representations, and the               specific solution obtained depends on optimization details. This non-            identifiability means that the objective alone does not uniquely determine       the PMI estimate and the solution actually realized depends on factors           such as initialization, training duration, etc.
    
    \item \textbf{Logistic PCA interpretation}: Landgraf and Bellay                        (2017)\cite{landgraf2017word2vec}showed that SGNS is equivalent to               \textbf{weighted logistic PCA}, where the responses are $x_{w,c} =               \frac{P(w,c)}{P(w,c) + kP(w)P(c)}$. Under this interpretation, the effect        of $K$ is not merely a shift but also influences the weighting of                observations.
\end{enumerate}

\subsection{Why Our Findings Are Not Contradictory But Complementary}

Our results do not contradict Levy and Goldberg (2014)\cite{levy2014neural}; rather, they reveal the behavior of SGNS in a regime their analysis did not cover.

\subsubsection{The Role of Dimensionality}

Levy and Goldberg's derivation assumes $d$ large enough for perfect 
reconstruction. Our experiments used $d = 100$, which is substantially 
smaller than the effective rank of the PMI matrix. Table~\ref{tab:dimension_effect} 
from Levy and Goldberg (2014) shows that SGNS's objective deviation 
drops from $39.3\%$ at $d=100$ to $7.13\%$ at $d=1000$ for $k=5$, 
demonstrating that dimension significantly affects solution quality.

\begin{table}[H]
\centering
\caption{SGNS objective deviation by dimension (Levy \& Goldberg, 2014, Table 1)}
\label{tab:dimension_effect}
\begin{tabular}{c c c c c}
\toprule
\multirow{2}{*}{$k$} & \multirow{2}{*}{PMI $-$ log $k$} & \multicolumn{3}{c}{SGNS} \\
\cline{3-5}
                     & & $d = 100$ & $d = 500$ & $d = 1000$ \\
\midrule
1  & 0\%     & 31.4\% & 29.4\% & 7.40\% \\
5  & 0\%     & 39.3\% & 36.0\% & 7.13\% \\
15 & 0\%     & 7.80\% & 6.37\% & 5.97\% \\
\bottomrule
\end{tabular}
\end{table}
Thus, our observed deviations ($\beta \approx 1.57$ vs. $1.00$) may 
diminish as $d$ increases, a hypothesis that can be tested in future work.

\subsubsection{The Dual Role of k}

An important insight from the peer review of the original Levy and Goldberg paper concerns the dual role of the parameter $k$ in SGNS \cite[Reviewer 2, Point 2]{nips2014reviews}:

\begin{enumerate}
    \item \textbf{Shift role}: As shown in the theoretical derivation, 
          $k$ shifts the target PMI matrix by $-\log k$.
    
    \item \textbf{Approximation role}: In practice, word2vec does not 
          marginalize over the entire vocabulary but draws $k$ samples. 
          This serves as a Monte Carlo approximation of the expectation, 
          and increasing $k$ improves approximation quality.
\end{enumerate}

In their author response, Levy and Goldberg acknowledged this distinction, 
noting that while the parameter was likely intended to improve the 
approximation, their mathematical analysis reveals that it also acts as 
a regularizer for the ratio between positive and negative examples 
\cite[Author Feedback]{nips2014reviews}.

Our experiments suggest that the amplified slope ($\beta \approx 1.57$) may 
reflect the interplay between these two roles under low-dimensional 
conditions — an aspect not captured by the asymptotic theory.

\subsection{Limitations and Future Work}

\subsubsection{Limitations of the Present Experiments}

\begin{enumerate}
    \item \textbf{Corpus size}: The Brown Corpus ($\approx 1$M tokens) is 
          modest by modern standards. Our results may not generalize to 
          larger corpora (e.g., Wikipedia, Common Crawl).
    
    \item \textbf{Fixed dimension}: All experiments used $d = 100$. 
          Higher dimensions ($d = 300, 500, 1000$) may exhibit behavior 
          closer to the theoretical prediction.
    
    \item \textbf{Single corpus}: Results may be corpus-dependent. 
          Replication on additional corpora would strengthen conclusions.
    
    \item \textbf{Word pair selection}: While we sampled across frequency 
          and PMI types, the selected pairs may not fully represent the 
          diversity of word relationships.
\end{enumerate}

\subsubsection{Directions for Future Research}

\begin{enumerate}
    \item \textbf{Dimension scaling}: Systematically vary $d$ (e.g., 
          $d = 50, 100, 200, 500, 1000$) to characterize how coefficients 
          $A$, $B$, and $C$ approach their theoretical values.
    
    \item \textbf{Corpus scaling}: Repeat experiments on larger corpora 
          (text8, Wikipedia) to examine whether the amplified slope 
          persists with more data.
    
    \item \textbf{Theoretical characterization}: Derive closed-form 
          expressions for $A(d)$, $B(d)$, and $C(d)$ as functions of 
          dimension and corpus statistics.
    
    \item \textbf{Practical guidelines}: Based on empirical characterization, 
          provide guidance for practitioners on choosing $k$ and $d$ 
          to balance bias and computational efficiency.
\end{enumerate}


\section{Conclusion}

This paper provided a statistical framework for understanding PMI estimation 
and its connection to SGNS word embeddings, combining theoretical derivation 
with empirical validation.

\subsection{Summary of Contributions}

\textbf{Theoretical contributions.} 
We formally proved the consistency, asymptotic unbiasedness, and asymptotic 
normality of the empirical PMI estimator under i.i.d. sampling, and derived 
an explicit variance expression via the Delta method. Building 
on this foundation, we applied the stochastic approximation CLT\cite{polyak1992acceleration} to SGNS for the first time, deriving a variance 
decomposition that separates data sampling variance from optimization 
variance:
\[
\boxed{\operatorname{Var}(\hat\theta_{\text{SGNS}}) = \underbrace{\frac{e^{-\theta}}{N p_x p_y}}_{\text{Data variance}} + \underbrace{\frac{C(\theta)}{K}}_{\text{Optimization variance}}}
\]

This decomposition reveals that the total estimation uncertainty has two 
irreducible components: one determined solely by corpus statistics and 
the true PMI, and another that decays as $1/K$ and can be reduced by 
increasing negative samples.

\textbf{Empirical Results.} 
Simulation experiments validated the Delta method approximation for PMI 
variance. Real-data experiments on the Brown Corpus ($d=100$) revealed 
that SGNS systematically deviates from the theoretical relationship 
$\hat\theta = \text{PMI} + \log K$ \cite{levy2014neural}. The empirical 
relationship is:

\[
\boxed{\hat\theta_{\text{SGNS}} \approx 0.70 \cdot \text{PMI} + 1.57 \cdot \log K + 1.57}
\]

Specifically:
\begin{itemize}
    \item The null hypothesis of unbiasedness was rejected for 96.3\% of test cases;
    \item The $\log K$ coefficient ($\hat\beta \approx 1.57$) significantly exceeds 
          the theoretical value $\beta = 1$ ($p < 0.001$), indicating amplification;
    \item A positive intercept ($\hat C \approx 1.57$, $p < 0.001$) confirms 
          systematic upward bias.
\end{itemize}

Critically, the failure to validate the variance decomposition under 
low-dimensional conditions does \textbf{not} diminish its theoretical value. 
Rather, it identifies the unbiasedness assumption as the key bottleneck, 
clarifies the gap between asymptotic theory and practice, and sets a reference point to measure the real-world deviation.

\subsection{Implications for Theory and Practice}

\textbf{For theory:} Our results demonstrate that the asymptotic result 
of Levy and Goldberg (2014) — while mathematically correct under ideal 
conditions (perfect reconstruction, infinite data) — does not directly 
generalize to practical low-dimensional training settings. The deviations 
we document (attenuated PMI coefficient, amplified $\log K$ effect, 
positive bias) characterize SGNS behavior in a previously under-explored 
regime, complementing rather than contradicting the original theory.

\textbf{For practice:} The variance decomposition provides a theoretical 
justification for the common practice of increasing $K$: it reduces 
optimization variance, but with diminishing returns. The empirical 
relationship suggests that practitioners should be aware that SGNS 
estimates under $d=100$ are systematically biased and sensitive to $K$, 
with bias growing approximately as $1.57\log K$.

\subsection{Final Remarks}

This study bridges statistical theory and empirical practice in word 
embedding estimation. The variance decomposition provides a new lens and perspective 
for understanding uncertainty in SGNS-based PMI estimation, while the 
empirical findings characterize the method's behavior under practical 
constraints. We hope this work encourages further dialogue between the 
statistical and NLP communities, leading to more principled use of 
word embedding methods in applications where uncertainty quantification 
matters.

\newpage
\bibliographystyle{plainnat} 
\bibliography{ref}


\end{document}